\documentclass{article}

\usepackage{microtype}
\usepackage{graphicx}
\usepackage{subcaption}
\usepackage{booktabs} 

\usepackage{hyperref}
\usepackage{colortbl}

\usepackage[accepted]{icml2026}

\usepackage{amsmath}
\usepackage{amssymb}
\usepackage{mathtools}
\usepackage{amsthm}

\usepackage[capitalize,noabbrev]{cleveref}

\theoremstyle{plain}
\newtheorem{theorem}{Theorem}[section]

\theoremstyle{definition}

\theoremstyle{remark}

\usepackage[textsize=tiny]{todonotes}

\usepackage{multirow}
\usepackage{tikz}
\usepackage{enumitem}
\usetikzlibrary{positioning,decorations,decorations.markings,arrows.meta,calc,shapes}

\usepackage{float}
\usepackage{mdframed}
\newtheorem*{restate}{Theorem} 
\usepackage{wrapfig}

\begin{document}

\twocolumn[
  \icmltitle{Fixed Aggregation Features Can Rival GNNs}



  \icmlsetsymbol{equal}{*}

  \begin{icmlauthorlist}
    \icmlauthor{Celia Rubio-Madrigal}{cispa}
    \icmlauthor{Rebekka Burkholz}{cispa}
  \end{icmlauthorlist}

  \icmlaffiliation{cispa}{CISPA Helmholtz Center for Information Security}

  \icmlcorrespondingauthor{Celia Rubio-Madrigal}{celia.rubio-madrigal@cispa.de}

  \icmlkeywords{Machine Learning, Graph Neural Networks, GNNs, Graph Learning, Kolmogorov-Arnold, Fixed Aggregation, Tabular Data}

  \vskip 0.3in
]



\printAffiliationsAndNotice{}  

\begin{abstract}
Graph neural networks (GNNs) are widely believed to excel at node representation learning through trainable neighborhood aggregations. We challenge this view by introducing Fixed Aggregation Features (FAFs), a training-free approach that transforms graph learning tasks into tabular problems. This simple shift enables the use of well-established tabular methods, offering strong interpretability and the flexibility to deploy diverse classifiers. Across 14 benchmarks, well-tuned multilayer perceptrons trained on FAFs rival or outperform state-of-the-art GNNs and graph transformers on 12 tasks—often using only mean aggregation. The only exceptions are the Roman Empire and Minesweeper datasets, which typically require unusually deep GNNs. To explain the theoretical possibility of non-trainable aggregations, we connect our findings to Kolmogorov–Arnold representations and discuss when mean aggregation can be sufficient. In conclusion, our results call for (i) richer benchmarks benefiting from learning diverse neighborhood aggregations, (ii) strong tabular baselines as standard, and (iii) employing and advancing tabular models for graph data to gain new insights into related tasks.
\end{abstract}

\section{Introduction}

Graph neural networks (GNNs) have become the standard approach for learning from graph based data, and in particular, for solving node classification. 
Most models follow the message-passing paradigm \citep{quachem}, where each node updates its representation by alternating neighborhood aggregation with learned linear combinations across multiple hops. 
This framework has been remarkably successful at combining node features with graph structure, driving applications in domains ranging from social networks to biology \citep{bongini2023biognn,Sharma_2024}. Yet, it comes at the cost of high model complexity that poses challenges for interpretation. We ask the question whether this high complexity is really necessary. 

Recent evidence \citep{luo2024classic} shows that classic models, such as GCN \citep{kipf2017semisupervised}, GAT \citep{velickovic2018graph}, and GraphSAGE \citep{graphsage}, surprisingly remain competitive when equipped with proper hyperparameter tuning and optimization techniques. When carefully tuned, they can rival more sophisticated approaches, including state-of-the-art Graph Transformers \citep{SGFormer,Polynormer,GOAT,NodeFormer,NAGphormer,GraphGPS,Exphormer} and models designed for heterophily \citep{H2GCN,CPGNN,GPRGNN,FSGNN,GloGNN}.
\nocite{luo2025can}

These results invite a closer look at which components of graph  architectures are essential for strong performance, and thus raise a natural next question:\ \emph{How relevant is {learning} the aggregation?} In fact, the field has invested in {learning} increasingly complex {convolution layers} and attention mechanisms.
In this paper we challenge that premise from first principles. Leveraging the Kolmogorov–Arnold representation theorem \citep{kolmogorov1957representation,SCHMIDTHIEBER2021119}, we give an explicit, lossless construction of neighborhood aggregations. Consequently, one can in theory encode neighbor features without discarding information. However, the same construction exposes a crucial gap between expressiveness and learnability:\ These lossless encoders are numerically brittle (e.g., sensitive to floating-point noise) and tend to produce ``rough" embeddings that are ill-suited for standard classifiers on Euclidean space such as MLPs.

\begin{figure*}
    \centering
    \resizebox{\linewidth}{!}{
    \begin{tikzpicture}[yscale=0.5]
        \node at (-0.5,0) {Base:\ $h_v^{(0)}=x_v$};
        \node at (-3,-1.5) {\begin{tabular}{c}
Node $v$'s \\ features
        \end{tabular}};
        \draw (0,-2) rectangle ++(1,1);
        \draw (-2,-2) rectangle ++(1,1);
        \draw (-1,-2) rectangle ++(1,1);
        \node at (1.5,-1.5) {$\oplus$};
        
        \begin{scope}[yshift=-3cm]
        \node at (-3,-1.75) {\begin{tabular}{c}
Neighbor \\ features
        \end{tabular}};
        \draw[dashed] (0,-2) rectangle ++(1,1);
        \draw[dashed] (-2,-2) rectangle ++(1,1);
        \draw[dashed] (-1,-2) rectangle ++(1,1);
        
        \begin{scope}[xshift=0.125cm,yshift=-0.25cm]
        \draw[dashed,fill=white] (0,-2) rectangle ++(1,1);
        \draw[dashed,fill=white] (-2,-2) rectangle ++(1,1);
        \draw[dashed,fill=white] (-1,-2) rectangle ++(1,1);
        
        \begin{scope}[xshift=0.125cm,yshift=-0.25cm]
        \draw[dashed,fill=white] (0,-2) rectangle ++(1,1);
        \draw[dashed,fill=white] (-2,-2) rectangle ++(1,1);
        \draw[dashed,fill=white] (-1,-2) rectangle ++(1,1);
        \end{scope}
        \end{scope}
        \end{scope}

        \begin{scope}[xshift=4cm]
        \node at (-0.5,0) {1$^{\text{st}}$ hop:\ $h_v^{(1)}$};
        \draw (0,-2) rectangle ++(1,1);
        \draw (-2,-2) rectangle ++(1,1);
        \draw (-1,-2) rectangle ++(1,1);
        \node at (1.5,-1.5) {$\oplus$};

        \draw[-{latex[scale=2]}] (-5.25,-5) to 
        node[midway,fill=white,circle,inner sep=1pt] {\large$\Phi$} 
        (-1.5,-1.5);
        \draw[-{latex[scale=2]}] (-4.25,-5) to node[midway,fill=white,circle,inner sep=1pt] {\large$\Phi$} 
        (-.5,-1.5);
        \draw[-{latex[scale=2]}] (-3.25,-5) to
        node[midway,fill=white,circle,inner sep=1pt] {\large$\Phi$} 
        (0.5,-1.5);
        \end{scope}
        
        \begin{scope}[yshift=-3cm,xshift=4cm]
        \node at (0,-3) {};
        \draw[dashed] (0,-2) rectangle ++(1,1);
        \draw[dashed] (-2,-2) rectangle ++(1,1);
        \draw[dashed] (-1,-2) rectangle ++(1,1);
        
        \begin{scope}[xshift=0.125cm,yshift=-0.25cm]
        \draw[dashed,fill=white] (0,-2) rectangle ++(1,1);
        \draw[dashed,fill=white] (-2,-2) rectangle ++(1,1);
        \draw[dashed,fill=white] (-1,-2) rectangle ++(1,1);
        
        \begin{scope}[xshift=0.125cm,yshift=-0.25cm]
        \draw[dashed,fill=white] (0,-2) rectangle ++(1,1);
        \draw[dashed,fill=white] (-2,-2) rectangle ++(1,1);
        \draw[dashed,fill=white] (-1,-2) rectangle ++(1,1);
        \end{scope}
        \end{scope}
        \end{scope}
        
        \begin{scope}[xshift=8cm]
        \node at (-0.5,0) {2$^{\text{nd}}$ hop:\ $h_v^{(2)}$};
        \draw (0,-2) rectangle ++(1,1);
        \draw (-2,-2) rectangle ++(1,1);
        \draw (-1,-2) rectangle ++(1,1);
        \node at (1.5,-1.5) {$\Bigg\rbrace$}; 

        \draw[-{latex[scale=2]}] (-5.25,-5) to 
        node[midway,fill=white,circle,inner sep=1pt] {\large$\Phi$} 
        (-1.5,-1.5);
        \draw[-{latex[scale=2]}] (-4.25,-5) to 
        node[midway,fill=white,circle,inner sep=1pt] {\large$\Phi$} 
        (-0.5,-1.5);
        \draw[-{latex[scale=2]}] (-3.25,-5) to 
        node[midway,fill=white,circle,inner sep=1pt] {\large$\Phi$} 
        (0.5,-1.5);
        \end{scope}
        
        \begin{scope}[yshift=-3cm,xshift=8cm]
        \node at (0,-3) {};
        \draw[dashed] (0,-2) rectangle ++(1,1);
        \draw[dashed] (-2,-2) rectangle ++(1,1);
        \draw[dashed] (-1,-2) rectangle ++(1,1);
        
        \begin{scope}[xshift=0.125cm,yshift=-0.25cm]
        \draw[dashed,fill=white] (0,-2) rectangle ++(1,1);
        \draw[dashed,fill=white] (-2,-2) rectangle ++(1,1);
        \draw[dashed,fill=white] (-1,-2) rectangle ++(1,1);
        
        \begin{scope}[xshift=0.125cm,yshift=-0.25cm]
        \draw[dashed,fill=white] (0,-2) rectangle ++(1,1);
        \draw[dashed,fill=white] (-2,-2) rectangle ++(1,1);
        \draw[dashed,fill=white] (-1,-2) rectangle ++(1,1);
        \end{scope}
        \end{scope}
        \end{scope}
        
        \begin{scope}[xshift=14.25cm]
        \node at (-3,-1.5) {\begin{tabular}{c}
Train MLP on \\ {$[h_v^{(0)}\oplus h_v^{(1)}\oplus h_v^{(2)}]$}
        \end{tabular}};
        \end{scope}
    \end{tikzpicture}
    }
    \caption{Fixed Aggregation Features (FAFs) are calculated as a pre-processing step, concatenated to the input ($\oplus$), and fed to an MLP. 
    If the {aggregation function} 
    $\Phi$ is injective, the neighborhood information is preserved. The Kolmogorov-Arnold 
    representation 
    theorem ensures the existence of such a function, although simple {reducers} are empirically more amenable for optimization. 
    }
    \label{fig:mainfig}
\end{figure*}
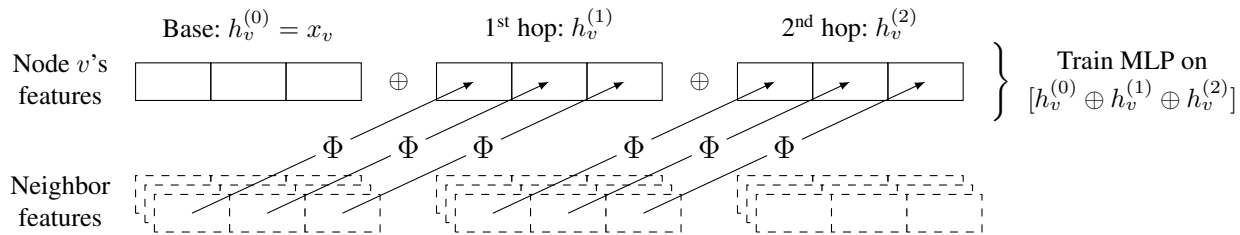

Surprisingly, we find that standard, untrained {aggregation operators}—like sum, mean, max, min, and std—though not information-preserving, yield useful features without any {learnable parameters}. 
Building on this observation, we propose Fixed Aggregation Features (FAF) (\S~\ref{s:FAF}):\ a training-free aggregation pipeline that applies fixed {aggregation functions—also referred to as ``reducers"—}over neighborhoods at multiple hops, concatenates the results into a tabular feature matrix, and then trains only a downstream classifier (e.g., an MLP). This data transformation brings several advantages:\ high interpretability (feature importance and ablations over hops/{reducers}), compatibility with the rich toolbox of tabular learning (designed to handle noise, class imbalance, feature selection, etc.), architectural flexibility, and reduced training compute.

Empirically, FAFs combined with well-tuned MLP classifiers are competitive on 12/14 common node-classification benchmarks, including citation \citep{Cora,Citeseer,Pubmed}, coauthor and Amazon co-purchase graphs \citep{shchur2018pitfalls}, Wikipedia \citep{wikics,heterophilous}, and other heterophilous datasets \citep{platonov2023critical}. Performance only really trails on Minesweeper and Roman-Empire, where the best GNNs rely on linear residual connections; in fact, 
the remaining gap aligns with the gains from the parameterized residuals reported by \citet{luo2024classic}. 
This pattern suggests that these datasets (Minesweeper and Roman-Empire) benefit from hop-specific aggregations or combinations of consecutive hops. 
Also, while these GNNs profit from many layers (10--15), the best-performing FAFs use only 2--4 hops.
On the other hand, why do FAFs perform on par for the other datasets?
Our theoretical analysis of the employed aggregation functions (\S~\ref{s:theory}) and our empirical findings (\S~\ref{s:experiments}) suggest that, for most benchmarks, the relevant signal is concentrated within hops 0–2; on hops 0–1, sum and mean preserve information. 
At higher hops, different {reducers} are complementary, but the information gain from min/max diminishes.

FAFs let us examine {datasets} from an optimization-first viewpoint without hard-to-interpret architectural factors (\S~\ref{s:advantages}). By turning neighborhoods into tabular features, we decouple representation from optimization and enable standard interpretability tools (e.g., feature importance) to identify which hop distances and {reducers} carry signal. Beyond revisiting the homophily–heterophily dichotomy or one-hop informativeness \citep{dichotomy,Zheng2024What,zheng2025let}, our method supports a richer characterization of interaction patterns:\ how signal varies across scales, which effects are additive vs.\ redundant, and where long-range dependencies appear. Moreover, gains reported by new {methods} could be dissected by constructing and comparing to their FAF counterparts. The tabular view also makes it natural to augment features, e.g., with network-science descriptors, or with neighborhood-masking features inspired by rewiring \citep{rubio-madrigal2025gnns,pmlr-v269-roth25a}. Furthermore, concurrent work has begun to explore a tabularization perspective to adapt tabular foundation models to graph data \citep{eremeev2025turning,choi2025tabpfncompetegnnsnode,hayler2025of}, a promising direction we help contextualize.

Together, our results suggest that many benchmarks do not require sophisticated learned aggregations, and that a large portion of GNN performance can be matched by powerful tabular predictors fed with fixed, transparent features.
Thus, FAFs can serve both as a \emph{strong baseline} and as a \emph{diagnostic tool} for graph benchmarks and methods. 
In summary, here are our main contributions:

    \textbf{1.} 
    \textbf{Theory:} We construct lossless neighborhood aggregations via Kolmogorov–Arnold representations, clarifying that learnability and numeric stability, not just expressiveness, govern practical success. We also analyze what information common {reducers} extract from neighborhoods, revealing information preservation at 0 and 1 hops and diminishing information with higher depth.

    \textbf{2.}
    \textbf{Method:} We introduce FAFs, which convert graph data into a tabular task by stacking fixed multi-hop aggregations, offering an interpretable framework to study the interplay between graph structure, features, and the task.

    \textbf{3.}
    \textbf{Empirical evidence:} FAFs match or exceed classic GNNs on 12/14 standard benchmarks. Our experiments corroborate our analysis, finding low-hop features to be more important and diminishing information at higher depth.

    \textbf{4.}
    \textbf{Implications:} Our findings question the necessity of learned neighborhood aggregation on current standard benchmarks, motivate strong tabular baselines for graph data, and open a path to more interpretable, efficient graph learning—and to designing harder benchmarks that genuinely benefit from learning aggregation.

\section{Related Work}

\textbf{Simplifying GNNs.}\quad
A growing body of work shows that much of a GNN's power can be retained—even improved for some tasks—when message passing is simplified or fixed. Early evidence comes from \citet{kipf2017semisupervised}, inspiring lines of work where aggregation layers are frozen or randomized:\ \citet{RR-GCN} are competitive on relational graphs with randomly transformed random features; \citet{Kelesis2025} analyze partially frozen GCNs, showing that fixing aggregation can mitigate over-smoothing and ease optimization; and GESN \citep{gesn} compute embeddings via a dynamical system with randomly initialized reservoir weights, after which only a linear readout is trained \citep{MICHELI2024127965}.

Another simplification comes from SGC \citep{sgcn}, which remove nonlinearities from a GCN, yielding a model equivalent to applying a low-pass graph filter followed by a linear readout. For link prediction \citep{qarkaxhija2024link}, SGCNs are found to be better than GCNs, but even removing the linear classifier\textemdash and thus all trainable parameters\textemdash provides a good baseline. 
Subsequent work has explored this idea further by precomputing propagation before training. 
SIGN \citep{sign_icml_grl2020} precomputes and concatenates features obtained from several fixed linear diffusion operators, then applies learned transformations to the representations---while discussing the need for exploring more expressive or wider local operators; GAMLP \citep{gamlp} and HOGA \citep{hoga} learn complex attention or gating mechanisms on them to decide which hops matter. 
These methods demonstrate separating propagation from prediction, though only as a way to improve scalability. Moreover, our FAFs are not limited to linear diffusions: reducers such as max, min, and std are nonlinear multiset operations, and several datasets require such reducers to reach their best performance.
Fixed aggregation is also used in APPNP \citep{gasteiger2018combining}:\ An MLP is first trained to produce node embeddings, and a subsequent Personalized PageRank–based propagation is applied; though the propagation itself is fixed, gradients flow through during backpropagation. And for graph {classification} on non-attributed graphs, \citet{cai2022simpleeffectivebaselinenonattributed} show that first-neighborhood statistics with an SVM form a surprisingly strong baseline, but performance lags on attributed graphs.

In contrast, we focus on {node classification} with rich node features. We aggregate and concatenate across {all} hops, and we place a more carefully tuned graph-agnostic classifier on top, all of which are necessary for our results\textemdash as shown in Tables~\ref{tab:testmain}, ~\ref{tab:logreg}, and \ref{tab:lasthop} (Appendix~\ref{app:mainexpperformance}, ~\ref{app:additional}).
We thus highlight the value of concatenating dependent but informative hop-wise features \citep{reddy2026when}, and the benefits of overparameterization—echoing evidence from random-feature models \citep{DONGHI2024128281}.

\textbf{Benchmarking GNNs.}\quad
Our work also connects to the growing literature on properly benchmarking GNNs and what constitutes a meaningful graph-learning dataset. Recent analyses warn that graph learning risks losing relevance without application-grounded benchmarks \citep{bechler-speicher2025position} and principled criteria for dataset quality beyond accuracy \citep{coupette2025no}. For graph {classification}, \citet{Errica2020A} show that, under controlled protocols, simple and even structure-agnostic baselines can rival complex GNNs, suggesting that architectures often fail to exploit graph structure. On the dataset side, \citet{bazhenov2025graphland} recently introduced industrial node property prediction benchmarks with graph-agnostic baselines. Their neighborhood feature aggregation (NFA), which augments tabular models with one-hop aggregated neighbor statistics, can be seen as a one-hop instance of our FAFs. We include preliminary results for four of their datasets in Appendix~\ref{app:graphland}.
In this context, FAFs serve as a simple stress test of whether proposed benchmarks genuinely benefit from learned message passing, and we argue that such well-tuned, fixed, multi-hop baselines should be routinely included when assessing new graph models and datasets.

\textbf{Tabularizing graphs for TFMs.}\quad
Concurrent work has explored converting node classification into a tabular prediction problem in order to leverage recent tabular foundation models (TFMs). \citet{eremeev2025turning} propose G2T-FM, which augments node features with one-hop NFA and structural embeddings before applying a tabular foundation model. \citet{choi2025tabpfncompetegnnsnode} similarly introduce TabPFN-GN, which constructs tabular features from node attributes, structural embeddings, and smoothed fixed aggregation. Finally, \citet{hayler2025of} uses ensembles of subsampled tables constructed from feature encodings, structural encodings, and labels. These works are similar in spirit to FAFs in that they expose graph information through tabular features, but differ in scope: they focus on adapting tabular foundation models for graph prediction, whereas our goal is to study how far simple fixed multi-hop aggregations with a carefully tuned graph-agnostic classifier can go, and what this implies about the need to learn neighborhood aggregation.

\textbf{GNN aggregation functions.}\quad
Classical message-passing GNNs differ primarily in how they aggregate neighbor features under permutation invariance. Sum, mean, and max are the canonical choices, with injectivity and stability trade-offs of each of them tied to their multiset representations \citep{xu2018how}. 
Beyond single operators, principal neighborhood aggregation (PNA) mixes several base {reducers} with degree-aware scalers to 
boost expressiveness and well-conditioning for continuous {reducers} \citep{Corso2020Principal}. 
Attention mechanisms instantiate learned weighted sums \citep{velickovic2018graph,brody2022how}, although it has been shown that they suffer from trainability problems, including small relative gradients on the attention parameters, slowed-down layer-wise training speed, and
the inability to mute neighbors \citep{mustafa2023are,mustafa2024dynamic,mustafa2024gate}.
Our perspective is complementary:\ We study fixed {reducers} whose strength comes from (i) their information preservation and (ii) their separability by a powerful downstream classifier.
This decoupling clarifies what must be learned (the readout) versus what can be fixed (the propagation), and it aligns with our empirical finding that stronger, well-tuned classifiers capitalize on rich, concatenated neighborhood views. Because of this, we argue that good optimization and learnability are as important as expressivity results. In line with this argument, \citet{NEURIPS2022_9e9f0ffc} show that, for tabular data, having the right embeddings for continuous features is key to closing the gap between transformer-like architectures and feed forward networks, proposing a lossless piecewise linear embedding to improve the trainability of the latter.

\textbf{Kolmogorov-Arnold theorem.}\quad
The Kolmogorov–Arnold representation \citep{kolmogorov1957representation} admits several equivalent formulations that reduce multivariate functions to compositions of univariate functions. 
Recent architectures explicitly instantiate such decompositions with learnable spline-based univariate components and linear mixing \citep{liu2025kan,carlo2025kolmogorovarnold}. In contrast, we use a specific fixed-aggregation formulation with predetermined aggregation weights and a fixed univariate encoding \citep{SCHMIDTHIEBER2021119}, so that any multivariate function $f$ can be learned exclusively from a single univariate readout $g$ applied to a fixed weighted sum of univariate transforms.

\section{Fixed Aggregation Features}\label{s:FAF}

Node neighborhoods can be compressed into single-node features, eliminating the need to learn feature embeddings before every layer of message passing. Our approach, Fixed Aggregation Features (FAFs), recursively constructs and concatenate features via a set of reducers $\mathcal{R}\subset\{\text{mean},\text{sum},\text{max},\text{min},\text{std},\ldots\}$ in the following way:
\begin{equation}\label{eq:fafs}
h_v^{(0,r)} = x_v,\quad
h_v^{(k,r)} = r\!\left(\{\,h_u^{(k-1,r)} : u\in N(v)\,\}\right),
\end{equation}
where $k\in\{1,\ldots,K\}$ and $r\in\mathcal{R}$. We then train an MLP on the concatenated representation 
\begin{equation}\label{eq:fafs-mlp}
z_v = x_v\oplus\left(\bigoplus{}_{r\in\mathcal{R}}\bigoplus{}_{\ k\in\{1,\ldots,K\}}\left(h_v^{(k,r)}\right)\right)
\end{equation}
with input dimensionality $|x_v|\cdot (1+|\mathcal{R}|\cdot K)$ per node $v$. Figure~\ref{fig:mainfig} illustrates the case for one reducer $\mathcal{R}=\{\Phi\}$ with $K=2$.  If the reducers are injective, then the neighborhood information at each depth is preserved in $z_v$. This waives the need for aggregating learned embeddings in GNNs, thus transforming graph data into high-dimensional tabular data. Our analysis explains why this is theoretically possible (\S~\ref{s:theory}). Additionally, in our experiments (\S~\ref{s:experiments}), we show that MLPs trained on FAFs can match the performance of classic GNNs on most standard node-classification benchmarks and, by comparing with \citet{luo2024classic}, of Graph Transformers and heterophily-aware models.

\subsection{Advantages of Tabular over Graph Representations}
\label{s:advantages}

\textbf{Interpretability.}\quad
Our construction concatenates each node’s original features with $K$-hop neighborhood statistics and feeds this expanded representation to a tabular classifier. This setup explicitly separates the feature and hop aggregation factors, allowing us to assess their individual contributions using the widely-used toolbox for tabular interpretability. 
For instance, we can analyze effects across hops by examining feature importance of the MLP.
We compute Shapley Additive Explanations (SHAP) \citep{shap} on Minesweeper with mean aggregation (Fig.~\ref{fig:shap-mine}), one of the two datasets where FAFs lag. 
In this dataset, labels are bombs, feature 0 masks other features, and features 1–6 {one-hot} encode the number of neighboring bombs. The top signal is the hop-1 mean of feature 1, i.e., the fraction of neighbors whose local bomb count is {null. When this proportion is greater than zero, the model knows that the node cannot have a bomb; when it is zero, all its neighbors observe bombs so there is a possibility of having a bomb.} This heuristic does not completely solve the problem, as neighbors {can be merely observing second-hop bombs}, creating ambiguity that likely causes residual errors. The model also correctly gives importance to the number of masked neighbors (hop-0 feature 0). These attributions clarify where the model succeeds and where it fails. We also report SHAP importances for two other datasets in Fig.~\ref{fig:shap} in Appendix~\ref{app:additional}:\ Pubmed (homophilic) and Amazon-Ratings (heterophilic).

\begin{figure}[b]
    \centering
    \includegraphics[width=\linewidth]{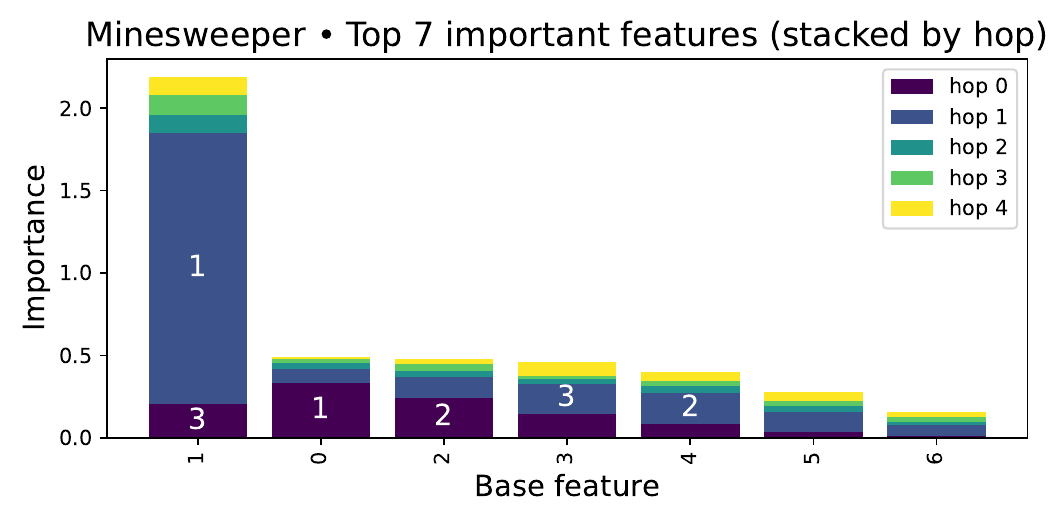}
    \caption{SHAP feature importance for Minesweeper, stacked by hop. Numbers show features' rankings per hop.}
    \label{fig:shap-mine}
\end{figure}

Moreover, rather than explaining a particular classifier, one can aim to explain the dataset, localizing which hops and features carry signal independent of model choice. Following \citet{donnelly2023the}, noisy tabular datasets often admit a ``Rashomon set'' of comparably well-performing models. Accordingly, feature importance is better assessed over this set—preferably constrained to simpler or sparser models—than from a single fit.
This lens offers a principled way to interrogate graph data beyond feature homophily–heterophily \citep{Zheng2024What} or strictly one-hop neighborhoods \citep{zheng2025let}. 

\textbf{Optimization.}\quad
GNNs usually exhibit early
overfitting, where training accuracy converges almost immediately while validation and test accuracy plateau or even decay. 
Thus, the best validation accuracy is often achieved before relevant aggregations are learned.
This might partially explain why FAFs can often compete with trained GNNs:\ They avoid overfitting aggregations.
GNNs can also suffer from ineffective aggregation learning \citep{mustafa2023are,mustafa2024gate}, so their potential to outcompete FAFs is likely underexplored due to trainability issues. 
By contrast, optimization on tabular data (like FAFs) is better tractable and understood by standard toolkits. 

\textbf{Efficiency.}\quad
Precomputing aggregation and then training an MLP is more scalable than repeatedly running message passing layers and backpropagating through them, as required in GNNs. However, as the number of reducers, original features, and hops in FAF increases, so does the input dimensionality, which in turn enlarges the parameter count of the MLP’s first layer. This issue could be mitigated through common feature reduction techniques. For the original features, we report the average training runtimes of our FAF and GNN models in Table~\ref{tab:runtimes} (Appendix~\ref{app:experimentalsetup}). FAFs are generally more efficient, particularly when using a single reducer.

\textbf{Feature augmentations.}\quad
Adding more features does not improve accuracy monotonically; beyond a point, some feature selection is needed. Still, the tabular view lets us concatenate diverse features atop the aggregations, which serve as additional informative covariates \citep{reddy2026when}. This can, in principle, include degree, centrality, or other node-level structural encodings \citep{10.1145/3770854.3785701}.
Our framework is also compatible with pre-processing graph rewiring, where aggregation is computed on a modified adjacency matrix to fight over-squashing \citep{topping2022understanding, jamadandi2024spectral}. But unlike standard rewiring, we can concatenate the rewired features instead of replacing the originals. 
As fixed aggregations can suffer from similar issues as trainable GNNs, FAFs can also benefit from proposed remedies. 
To extract more precise information from complex environments, we examine a feature similarity-based rewiring based on \citet{rubio-madrigal2025gnns}, 
where edges of negative feature cosine similarity between nodes are dropped. 
We then append features aggregated on the rewired graph, or split edges into positive/negative sets and aggregate separately, inspired by computational-graph splitting that fights over-smoothing \citep{pmlr-v269-roth25a}.
Results in Table~\ref{tab:rewiring} (Appendix~\ref{app:additional}) show that on most datasets, concatenating these extra features yields larger gains than using them as substitution.
These augmentations not only improve performance but also help disentangle where the gains come from:\ additional extracted signal versus changes to the optimization of graph models. This is akin to analyses of SGC and GESN \citep{MICHELI2025134}, though in our case we obtain benefits from the operations.
We therefore advocate FAF constructions as {baselines for methods that modify the aggregation component of GNNs}.

\section{
Does Aggregation Need Learning?}\label{s:theory}
Let $G=(V,E)$ be a graph with node features ${\mathcal{X}}\in\mathbb{R}^{F}$. A neighborhood function is a map $f:\mathcal{M}(\mathbb{R}^{F})\to\mathbb{R}$, where $\mathcal{M}(\mathbb{R}^{F})$ denotes the set of finite multisets of vectors in $\mathbb{R}^{F}$, acting on the neighborhood multiset $X_v:=\{\mathcal{X}_u:u\in N(v)\}$ for $v\in V$. We treat $v\in N(v)$ when self-loops are added, depending on the setup. 
We seek a fixed encoder $\Phi$ such that we can learn any neighborhood function via a univariate map $g$ with $f=g\circ\Phi^{-1}$. This enables tabular classifiers to learn graph data losslessly.

\textbf{What is preserved by standard aggregations?}\quad
FAFs apply a transformation of neighborhoods that is not learnable, which raises the question of what information gets lost by the aggregation. 
Permutation invariant aggregations treat graph neighborhoods as multisets consisting of feature vectors of neighbors.
Accordingly, they extract distribution information and forget about the identity of specific neighbors.
This property is usually regarded as helpful inductive bias and therefore of no concern. 
Our next theorems analyze which information is preserved by sum and mean aggregations from these multisets.
To do so, we first generalize Lemma~5 by \citet{xu2018how} for one-hot encoded discrete features to orthogonal features.
Combined with the fact that hop features are concatenated, this insights establishes that information from the 1-hop neighborhood can be preserved.

\begin{theorem}[1-hop aggregation] \label{thm:onehop}
Assume the features $\mathcal{X}$ are orthogonal. 
Then, the function $h(X) = \sum_{x \in X} x$ {defined on} multisets $X \subseteq \mathcal{X}$ of bounded size is {injective}.
Moreover, any multiset function $f$ can be decomposed as $f(X) = g\!\left(\sum_{x \in X} x\right)$ for some function $g$.
\end{theorem} 
The proof is given in Appendix~\ref{app:one-hop}. 
Note that a multiset $X \subseteq \mathcal{X}$ is characterized by the count $n_x$ of elements that have specific features $x$.
These counts can also be extracted from the sum $h(X)$ (as demonstrated in the proof).
Consequently, any multiset function $f$ transforms such counts by $f(n_x)$.
The function $g$ would thus first extract the counts from the sum $h(X)$ and then apply $f$ to the counts.
If the features of a node $v$ include its degree $d_v$, then mean aggregation contains the same information, as a classifier can learn to multiply $h(X) = 1/d_v\sum_{x \in X} x$ by $d_v$.
In contrast, max and min aggregations extract whether at least one neighbor has a specific feature property.
They focus on the tails of distributions rather than full neighborhoods.

\textbf{Information loss for $k$-hops.}\quad
One might hope that the above theorem also applies to aggregations from hop $k$ to $k+1$.
The orthogonality assumption, however, is essential and no longer met by the aggregated neighbor features ${h^{(k)}_n}$ for $k \geq 1$.
As a consequence, from $k \geq 2$, not all information about the distribution of features across neighbors is preserved, as Figure~\ref{fig:tree-example} exemplifies in Appendix~\ref{app:second-hops}.
In particular, ${h^{(2)}_1}$ captures neither the degrees of its neighbors, nor the number and types of second-hop neighbors associated with each first-hop neighbor. 
Even so, aggregation still extracts useful information, and different aggregations concatenate complementary properties of neighborhoods:

1.
{Sum aggregation:}
Sums count, for each of the $n$ distinct orthogonal feature vectors $x_f$, how many nodes in the $k$-hop neighborhood exhibit feature $f$. 
A classifier can extract it  by computing $x^T_f h^{(k)}_v$.
Note that nodes reachable through multiple length-$k$ paths are counted multiple times.

2. 
{Mean aggregation:}
Means can partially distinguish neighbors with different degrees by considering the fraction of nodes that exhibit a specific feature vector $x_f$. The quantity
$x^T_f h^{(k)}_v$ weights each node $i$ with feature $x_f$ by $1/d_i$.
Note that nodes reachable through multiple length-$k$ paths are again counted with multiplicity.

3. 
{Max/Min aggregation:}
Max aggregation on one-hot encoded features returns whether at least one node within $k$ hops has a given feature. For large neighborhoods as $k$ increases, this indicator quickly saturates, so increasing hops adds little further information.
The same reasoning applies when taking the maximum entry of the orthogonal features.
It also applies to the minimum:\
It indicates whether any node within $k$ hops lacks the feature, and increasing $k$ adds little further information.

\textbf{Lossless neighborhood aggregation.}\quad
When node features are real-valued in general, \citet{Corso2020Principal} show that no single \emph{continuous}, permutation-invariant aggregation function can be lossless for all multiset functions. This mirrors a classical topological obstruction (Netto, 1879):\ There is no continuous bijection $\mathbb{R} \!\to\! \mathbb{R}^2$ \citep{DAUBEN1975273}. 
However, there can exist discontinuous bijections, namely space filling functions. 
We adopt a concrete construction based on ternary expansions and the Cantor set, adapted from a Kolmogorov-Arnold representation variant from Theorem~2 by \citet{SCHMIDTHIEBER2021119}.

\begin{theorem}[Thm.~2 of \citet{SCHMIDTHIEBER2021119}]
\label{thm:ka}
For any fixed $d \ge 2$, there exists a monotone function $\phi:[0,1]\to\mathcal{C}$ (the Cantor set) such that the map
$
\Phi(x_1,\dots,x_d)\;=\;3\sum_{p=1}^{d} 3^{-p}\,\phi(x_p) 
$
is injective on $[0,1]^d$. Moreover, for every continuous $f:[0,1]^d\to\mathbb{R}$ there exists a continuous $g:\Phi([0,1]^d)\to\mathbb{R}$ with $f(x_1,\ldots,x_d) \;=\; g\!\left(\Phi(x_1,\ldots,x_d)\right).$

\end{theorem}

Theorem~\ref{thm:ka} isolates all required discontinuity into a \emph{fixed} aggregation. 
While $\Phi$ is not continuous, its inverse is, which makes the learnable $g := f \circ \Phi^{-1}$ inherit the continuity properties of $f$.
\citet{SCHMIDTHIEBER2021119} has also quantified how much information is lost if $g$ is learned instead of $f$.
For $f$ $\beta$-smooth with $\beta\leq1$, there is no difference in the approximation rate. However, for higher order smoothness, the multivariate and univariate function approximation may vary. Note that this aggregation even remembers node identities. From this theorem we can learn the following insight:

\begin{mdframed}[
  innerrightmargin=4pt,
    innerleftmargin=4pt,
  ]
A lossless, fixed, even univariate neighborhood aggregation function exists, but it has to be discontinuous for general continuous features. 
\end{mdframed}

\textbf{Implications and open challenges.}\quad
When we encode neighborhoods via $\Phi$ and learn $g$ so that $f = g \circ \Phi ^{-1}$, the information content, smoothness, and approximation rates of the neighborhood function $f$ transfer to $g$.
However, this theoretical sufficiency does not guarantee strong empirical performance when $\Phi$ is used as a {reducer} for FAF (see Table~\ref{tab:fafka}, Appendix~\ref{app:additional}). In Appendix~\ref{app:cantor}, we visualize how $\Phi$ maps 2D circles into the univariate Cantor set, and how $\Phi^{-1}$ can recover them continuously. We also compare against mean and std, and observe that $\Phi$ pushes inputs that are close together into far-apart representations, whereas mean and std bring together far-apart inputs that share commonalities. It is the case that, in practice, the simple statistics studied above often provide distributional summaries that downstream classifiers exploit more effectively.

An ideal aggregation would be both injective, like $\Phi$, and would extract useful statistical insights, like mean. 
It is still an open challenge to design, or potentially learn, efficient embeddings that extract relevant information from neighborhoods, while easing the learning problem for the classifier \citep{uniExist}.
One might expect GNNs to learn such representations end-to-end without overfitting.
Our experiments with FAFs (Table~\ref{tab:testmain2}) suggest—despite some information loss at iterative hops—that simple {reducers} suffice for most standard node-classification benchmarks. 

Experimentally, we find that mean aggregation alone is often among the top performers.
This suggests that neighborhood distributions provide most task-relevant signal, and that neighbor degrees encode useful structural information, helping to distinguish the contribution of distinct neighbors.
Consistent with this, the most relevant information is already provided in the immediate neighborhood ($k=1$--$2$, see Table~\ref{tab:locallayers} on increasing hops) and the concatenation of this information is key so that it is not lost by repeated aggregations (see Table~\ref{tab:lasthop} on only using the last hop).
Consequently, information loss at larger $k$ 
is of little practical concern—except for two datasets that appear to require subtler information. 
Taken together, these observations motivate the following hypothesis:
\begin{mdframed}[
  innerrightmargin=4pt,
    innerleftmargin=4pt,
  ]
Hypothesis:\ For most standard node-classification benchmarks, either the predictive signal is already concentrated within the first one or two hops, or current GNNs struggle to learn layer-wise aggregations that extract relevant information beyond mean or sum.
\end{mdframed}

The first part of this hypothesis underscores the need for more real-world datasets where long-range interactions and richer aggregations matter, supporting prior calls to revisit benchmark design \citep{Errica2020A,bechler-speicher2025position}. Although some tasks like Roman Empire benefit from long-range signal, making deep graph models work reliably remains a challenge. Recent evidence \citep{zhou2025glora} indicates that graph models generally struggle to capture interactions beyond 13 hops, irrespective of over-smoothing, over-squashing, or vanishing gradients.

The second part of the hypothesis concerns the ability of GNNs to actually realize useful aggregations in practice. For instance, GNNs may not move far enough from their initializations. Indeed, the two datasets on which GNNs hold an advantage require linear residual transformations to realize that gap \citep{luo2024classic}. Prior work also shows that GATs cannot flexibly adjust attention to shut off unhelpful neighbors \citep{mustafa2024gate}. This supports our results on rewiring the adjacency matrix before aggregation, as shown in \S~\ref{s:advantages}. If GATs could learn to prune the edges that we manually drop, they would enjoy similar gains. 

\textbf{Future work.}\quad
We see opportunity for future work along three fronts that build directly on our findings: 
    
    1. 
    Feature/reducer engineering: FAFs highlight untapped potential for designing meaningful node features that encode graph structure, require less learning, potentially preserve more—but ideally only relevant—information, and allow for higher learning efficiency. In combination with partial feature learning, they might form the basis of a new generation of graph based learning architectures.
    
    2. 
    Moving with and beyond injectivity: As our theory and experiments highlight, improving GNN expressiveness and thus injectivity alone is not likely to inspire practical improvements on current benchmarks, as those can be competitively solved with simple, non-injective aggregation. We therefore call for a shift in focus from mere injectivity to other learning properties—a gap to be addressed not only for FAFs but for GNNs in general.
    
    3.
    New benchmarks: Enough information to solve current benchmark tasks is already contained in early hops and can be extracted with simple, non-injective aggregation. If we really want to showcase the capabilities of GNNs to learn meaningful features, we need more difficult benchmarks that require this ability.

\begin{mdframed}[
  innerrightmargin=4pt,
    innerleftmargin=4pt,
  ]
In theory, fixed information-preserving aggregations can reduce graph learning to tabular prediction. In practice, task relevant representations and information preservation are a challenge. Progress likely requires both more amenable {reducers} and better tasks for evaluation.
\end{mdframed}

\section{Experiments}\label{s:experiments}

\begin{table*}[t]
\caption{Test accuracy on node classification: FAFs against classic GNNs.}
\label{tab:testmain2}
\centering\resizebox{0.95\linewidth}{!}{%
\begin{tabular}{lccccccc}
\toprule
Dataset & computer & photo & ratings & chameleon & citeseer & coauthor-cs & coauthor-physics \\
\midrule
GCN & 93.58 ± 0.44 & 95.77 ± 0.27 & 53.86 ± 0.48 & \textbf{44.62 ± 4.50} & \textbf{72.72 ± 0.45} & 95.73 ± 0.15 & \textbf{97.47 ± 0.08} \\
GAT & \underline{93.91 ± 0.22} & \underline{96.45 ± 0.37} & \textbf{55.51 ± 0.55} & 42.90 ± 5.47 & \underline{71.82 ± 0.65} & \underline{96.14 ± 0.08} & \underline{97.12 ± 0.13} \\
SAGE & 93.31 ± 0.17 & 96.17 ± 0.44 & \underline{55.26 ± 0.27} & 43.11 ± 4.73 & \underline{71.82 ± 0.81} & \textbf{96.21 ± 0.10} & 97.10 ± 0.09 \\
FAF$_{\rm best val}$ & \textbf{94.01 ± 0.21} & \textbf{96.54 ± 0.13} & 55.09 ± 0.24 & \underline{42.96 ± 2.45} & 70.48 ± 1.24 & 95.37 ± 0.17 & 97.05 ± 0.18 \\
\bottomrule
\end{tabular}%
}

\medskip 

\centering\resizebox{0.95\linewidth}{!}{%
\begin{tabular}{lccccccc}
\toprule
Dataset & cora & minesweeper & pubmed & questions & roman-empire & squirrel & wikics \\
\midrule
GCN & \textbf{84.38 ± 0.81} & \underline{97.48 ± 0.06} & \underline{80.00 ± 0.77} & \underline{78.44 ± 0.23} & \textbf{91.05 ± 0.15} & \underline{44.26 ± 1.22} & 80.06 ± 0.81 \\
GAT & 83.02 ± 1.21 & 97.00 ± 1.02 & 79.80 ± 0.94 & 77.72 ± 0.71 & 90.38 ± 0.49 & 39.31 ± 2.42 & \textbf{81.01 ± 0.23} \\
SAGE & \underline{83.18 ± 0.93} & \textbf{97.72 ± 0.70} & 77.42 ± 0.40 & 76.75 ± 1.07 & \underline{90.41 ± 0.10} & 40.22 ± 1.47 & \underline{80.57 ± 0.42} \\
FAF$_{\rm best val}$ & 82.84 ± 0.63 & 90.00 ± 0.39 & \textbf{80.96 ± 1.06} & \textbf{78.69 ± 0.50} & 78.11 ± 0.38 & \textbf{44.59 ± 1.62} & 80.25 ± 0.34 \\
\bottomrule
\end{tabular}%
}
\end{table*}

\textbf{Setup.}\quad
We report test
performance for GCN \citep{kipf2017semisupervised}, GAT \citep{velickovic2018graph}, and GraphSAGE \citep{graphsage} versus our approach, which feeds our Fixed Aggregation Features (FAFs) into MLPs (Figure~\ref{fig:mainfig} and Eqs.~\ref{eq:fafs}, \ref{eq:fafs-mlp}), in Table{~\ref{tab:testmain2}. We aggregate up to the same hop as the GNN baselines. {We obtain the best FAF variant from validation results, which are included in Table~\ref{tab:valmain} (Appendix~\ref{app:mainexpperformance}).} FAF$_4$ uses the reducers \{mean, sum, max, min\}, and is tuned with the hyperparameter grid from \citet{luo2024classic}; this makes our results directly comparable to their Graph Transformers and heterophily-aware architectures, where they find that classic GNNs can also rival them. Additional FAF variants include mean+std, mean only, max+std, max only, sum only, and std only. The best overall result is shown in \textbf{bold}, the second best is \underline{underlined}. More details on the experimental setup are given in Appendix~\ref{app:experimentalsetup}. Detailed accuracy values of all FAF variants are in Appendix~\ref{app:maintables}. Training, validation, and test curves of all datasets for FAF$_4$ and GCN are shown in Fig.~\ref{fig:curves} in Appendix~\ref{app:curves}. Ablation results are included in Appendix~\ref{app:additional}. 

\textbf{Comparison to classic GNNs.}\quad
Overall, we improve on 5 datasets, match within error or {1\% on another 5}, and trail on {4}. On most datasets, FAF$_4$ performs comparably to mean+std. {Among the ones within 1\%, we have} Coauthor-CS and Coauthor-Physics (Figs.~\ref{fig:curve-coauthor-cs},~\ref{fig:curve-coauthor-physics}), which are the largest and most feature-rich; targeted feature selection may close the gap. Among the 4 trailing datasets, two are homophilic and two heterophilic; the homophilic tasks are close to parity. Citeseer exhibits optimization instability (Fig.~\ref{fig:curve-citeseer}), {and Cora has a large test-validation gap in GCNs (Fig.~\ref{fig:curve-cora})}, not present in any other. The two heterophilic datasets, Minesweeper and Roman-Empire (Figs.~\ref{fig:curve-minesweeper},~\ref{fig:curve-roman-empire}) show larger performance drops. This behavior mirrors the decrease reported by \citet{luo2024classic} when residual connections are removed. Notably, the best-performing FAFs on these two datasets use far fewer hops (4 and 2) than the GNN baselines (15 and 10), suggesting that key signal lies at longer ranges. The shallower FAFs under-aggregate relative to what those tasks require, but adding extra hops does not provide extra information, as discussed in \S~\ref{s:theory}. We also show it in Table~\ref{tab:locallayers}, where we concatenate up to different amount of hops; we show its validation counterpart in Table~\ref{tab:locallayersval} (Appendix~\ref{app:additional}). In fact, many datasets peak at $k=2$, and either plateau or decrease in performance. These results also suggest that including later hops in any graph learning method may provoke overfitting, which can be an alternative explanation for the degradation of performance on deep GNNs apart from over-smoothing---as here it is not variable by construction. We also include an MLP with $k=0$, which performs worse than the other models.

\textbf{Best hyperparameters.}\quad
We include the best hyperparameters found for FAF$_4$ in Table~\ref{tab:hyperparams_per_dataset} (Appendix \ref{app:experimentalsetup}). All FAF variants benefit from normalization components, as aggregated features can vary widely in scale across reducers and hops. Compared with GNNs, FAFs typically favor larger learning rates, which can yield faster training, improved generalization via implicit regularization, and feature sparsity \citep{MohtashamiJS23,NEURIPS2024_6b737522}. Dropout levels, however, are broadly similar to those used for GNNs. This suggests that dropout’s gains on these node-classification tasks are driven more by dataset properties than by the specifics of training graph convolutions, which nuances prior interpretations \citep{luo2025beyond}.

\textbf{Ablation on single {reducers}.}\quad
Concatenating multiple aggregations has advantages and drawbacks. On the plus side, an MLP can learn to weight each reducer, removing the need to pick one per dataset. Because our individual reducers are not lossless, different datasets may favor different ones; moreover, adding informative, correlated covariates can improve robustness and reduce variance \citep{reddy2026when}. On the downside, concatenation increases input dimensionality, with corresponding memory and optimization costs. Table~\ref{tab:testmain} (Appendix~\ref{app:mainexpperformance}) reports test results when using a single aggregation at a time. {Note that this resembles a simple feature selection over FAF$_4$.} We keep the same hyperparameters as FAF$_4$ to isolate the effect of the aggregation choice, though the lower dimensionality could allow for alternative settings that further improve performance. Surprisingly, a single {reducer} often suffices, though the preferred choice varies by dataset. The mean is most frequently strongest, sometimes surpassing FAF$_4$ and FAF$_{\rm {mean},\rm {std}}$—i.e., on Amazon-Computer and Amazon-Photo. This may reflect optimization challenges from high-dimensional inputs or increased overfitting. Still, mean is not universally best, as sum and max win on some datasets (e.g., Citeseer favors sum; Amazon-Ratings favors max). Combining {reducers} therefore remains beneficial when one wishes to avoid committing to a specific one a priori.

\textbf{Comparison to one-layer classifier and last hop.}\quad
Some prior simplifications of GNNs \citep{sgcn,MICHELI2024127965} effectively fix the aggregation and train a single linear layer on the final-hop representation. In contrast, we concatenate representations from all hops and train a well-tuned MLP classifier. This choice is crucial for matching GNN performance. As shown in Table~\ref{tab:logreg} (Appendix~\ref{app:additional}), MLPs consistently outperform a single linear layer applied to the same concatenated features, indicating that their nonlinearity and increased capacity are important to learn from multi-hop features. Moreover, Table~\ref{tab:lasthop} (Appendix~\ref{app:additional}) shows that only using the last hop lacks important information that is not transmitted across aggregations.

\textbf{Kolmogorov-Arnold aggregation.}\quad
Our hypothesis that Roman-Empire lags due to information loss is reinforced by a FAF that uses the Kolmogorov–Arnold (KA) function $\Phi$, which is theoretically lossless (Theorem~\ref{thm:ka}). As we exemplified in Fig.~\ref{fig:cantor}, in practice KA is hard to use for classification. We observe in Table~\ref{tab:fafka} (Appendix~\ref{app:additional}) that in some datasets it struggles to fit during training, while others show mild overfitting. Nevertheless, in Roman-Empire this obtains a test accuracy of $80.33 \pm 0.47$, the highest among all FAFs, suggesting that providing full neighborhood information helps close the gap in this task. This, in turn, highlights the need for benchmarks where predictive signal genuinely arises at distant hops in complex ways.

\section{Conclusions}
We have introduced Fixed Aggregation Features (FAFs), a non-learnable tabular mapping from local neighborhoods of graph features to univariate representations that an MLP can learn to classify. Our analysis shows that fixed, injective neighborhood aggregation functions exist, linking multiset expressivity to Kolmogorov–Arnold factorizations; thus learned message passing is not required for expressivity in theory. But in practice, common non-injective {reducers} (mean, sum, max, min) train more reliably, underscoring a gap between what is expressive in principle and what is reliably learnable.   
We also highlight the practical advantages of a tabular view, such as access to the rich tabular toolkit of interpretability and tuning, and isolated representation from inference so we can attribute gains or failures to the features themselves rather than to message-passing optimization. 

On node classification datasets, FAFs are a strong baseline:\ They match or beat classic GNNs on many benchmarks, and trail only on two datasets needing longer-range interactions, where residualized GNNs help. 
Two ablations explain most gains:\ A well-tuned MLP beats a single linear classifier on top of FAF features, and concatenating all hop beats using only the last hop. This is consistent with later hops losing detail for these practical aggregat{ion schemes}. Surprisingly, two hops usually suffice, suggesting either limited signal in current benchmarks, or difficulty training deep GNNs to exploit more of it. While our theory carries over other downstream tasks, other benchmarks may surface different constraints that could alter the empirical outcomes.

Our findings have immediate implications. We recommend always including a tuned FAF baseline in future studies to calibrate what fixed aggregation alone can achieve; re-evaluating—and, when appropriate, retiring—datasets on which FAFs reach state-of-the-art performance; and developing benchmarks that genuinely require long-range dependencies and inter-hop dynamics. More broadly, we advocate for simplifying models and balancing expressiveness against optimizability, rather than assuming that extra parameters or higher expressiveness extract more relevant signal than simple baselines. Notably, several phenomena that are often blamed on graph architectures—overfitting, depth-related degradation, and sensitivity to dropout—also arise in tabular settings, indicating that some limitations may stem from dataset properties rather than the graph-aware architectures alone.

\section*{Acknowledgments and Disclosure of Funding}
The authors gratefully acknowledge the Gauss Centre for Supercomputing e.V. for funding this project by providing computing time on the GCS Supercomputer JUWELS at Jülich Supercomputing Centre (JSC). We also gratefully acknowledge funding from the European Research Council (ERC) under the Horizon Europe Framework Programme (HORIZON) for proposal number 101116395 SPARSE-ML.

\section*{Impact Statement}

This paper presents work whose goal is to advance the field of Machine
Learning. There are many potential societal consequences of our work, none
which we feel must be specifically highlighted here.

\bibliography{example_paper}
\bibliographystyle{icml2026}

\newpage
\appendix
\onecolumn

\section{Study of Reducers for Neighborhood Aggregation}

\subsection{Kolmogorov-Arnold Function $\Phi$ and its Continuous Inverse}\label{app:cantor}

\begin{figure*}[h!]
    \centering
  \begin{subfigure}{\linewidth}
    \centering
    \includegraphics[width=\linewidth]{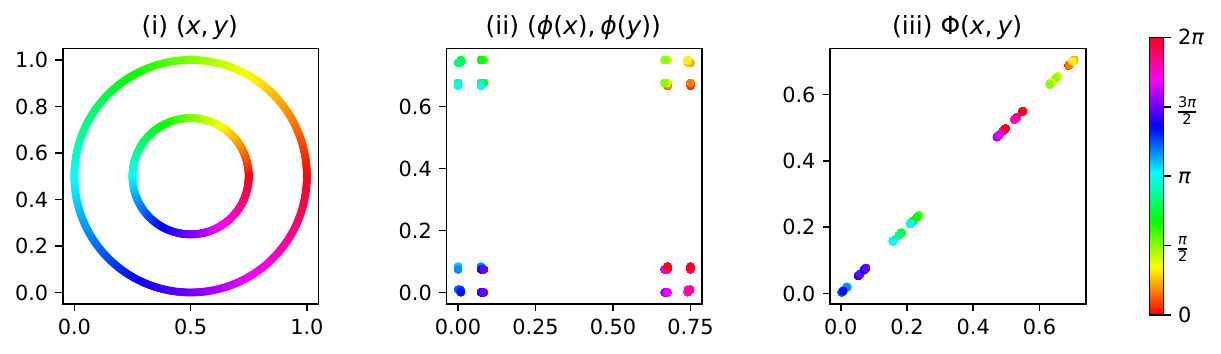}
    \subcaption{Aggregating from $(x,y)$ to $(\phi(x),\phi(y))$ to $\Phi(x,y)$.}\label{fig:cantor-a}
  \end{subfigure}
  
  \begin{subfigure}{\linewidth}
    \centering
    \includegraphics[width=1\linewidth]{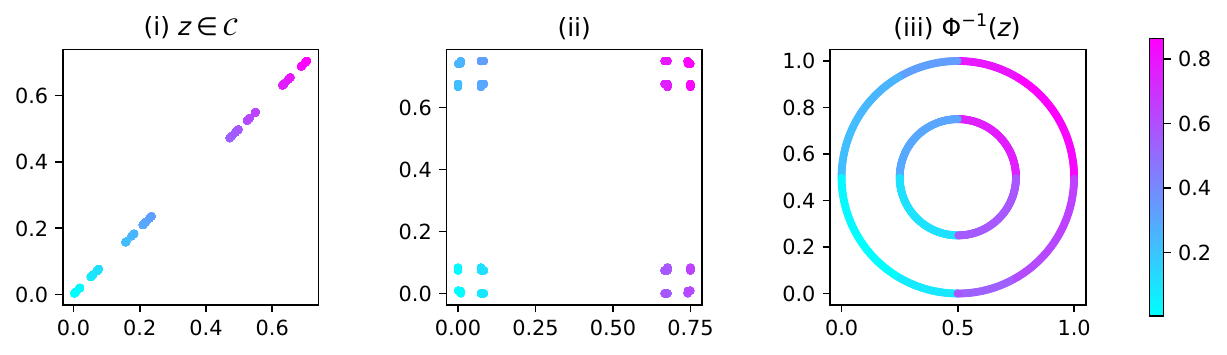}
    \subcaption{Recovering information from the aggregated variable $z\in\mathcal{C}$ to $\Phi^{-1}(z)$.}\label{fig:cantor-b}
  \end{subfigure}
  
  \begin{subfigure}{\linewidth}
    \centering
    \includegraphics[width=1\linewidth]{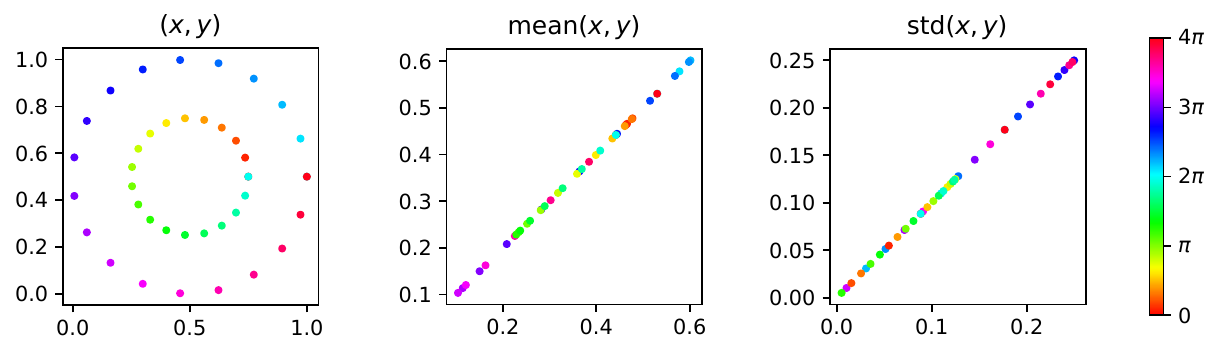}
    \subcaption{Mean and std { as aggregation functions}.}\label{fig:cantor-c}
  \end{subfigure}
    \caption{Functions $\Phi$ (Thm~\ref{thm:ka})\textemdash and its inverse\textemdash, mean and std. Circles and square-like panels (a.i, a.ii, b.ii, b.iii, c.i) live in the 2D space, while segments and Cantor sets (a.iii, b.i, c.ii, c.iii) live in 1D. Colors in (a) and (c) are based on angles on 2D, while colors in (b) are based on position.}
    \label{fig:cantor}
\end{figure*}

Here we showcase the behavior of $\Phi$ from the Kolmogorov-Arnold representation from Theorem~\ref{thm:ka}. In Fig.~\ref{fig:cantor-a}, we see the $\Phi$ image of two circles colored by their angle. Colors that were close together end up in separate parts of the Cantor set; for instance, oranges and reds, or purples and blues. In contrast, in Fig.~\ref{fig:cantor-b} we see $\Phi^{-1}$ maps the Cantor set to the circles in such a way that all colors maintain their closeness. 

If we use $\Phi$ as a fixed neighborhood aggregation, the classifier on top needs to learn to reverse it, therefore it is advantageous to have a continuous inverse. However, this does not give information about neighborhood distributions like the commonly used mean, sum, and max. In Fig.~\ref{fig:cantor-c} we show the behavior of mean and std; mean gives approximate location but fuses together points that are very far apart. For instance, blues and reds have an unusually large first (but different) coordinate and are mapped to the center; however, this information can be recovered with std.

The fact that the injective aggregator derived from theory is not practically useful is itself an insight. Our findings suggest that injectivity, and thus whether information is preserved by the aggregation, may not be the right lens through which to explain GNN performance. What seems to matter more is the model's ability to learn to discard less relevant information as depth increases. Preserving all neighborhood information across many hops may simply not be tractable or useful.

\subsection{Proof of Main Theorem}\label{app:one-hop}
For convenience, the following theorem restates Theorem~\ref{thm:onehop} of the main paper. 
Here, ``orthogonal features'' refers to the distinct feature states
that occur in the graph: if
\(\mathcal X=\{x_1,\ldots,x_n\}\subset \mathbb{R}^F\) denotes these
states, then \(x_i^\top x_j=0\) for \(i\neq j\). This does not require
\(x_u^\top x_v=0\) for every pair of distinct nodes \(u\neq v\).

\begin{restate}[1-hop aggregation] 
Assume the features $\mathcal{X}$ are orthogonal. 
Then, the function $h(X) = \sum_{x \in X} x$ {defined on} multisets $X \subseteq \mathcal{X}$ of bounded size is {injective}.
Moreover, any multiset function $f$ can be decomposed as $f(X) = g\!\left(\sum_{x \in X} x\right)$ for some function $g$.
\end{restate} 

\begin{proof}
Note that the multiset $X$ is fully characterized by the number $n_f$ of nodes in the set that have a feature $x_f$ for all possible features $x_f$.
Our objective is to show that this information is contained in the aggregated form $h(X) = \sum_{x \in X} x$.

Let us assume that the features are orthogonal. 
Accordingly, the features $x_v$ of each node $v$ assume one of a finite number of possible states $x_1, \cdots x_n \in  \mathbb{R}^{n_f}$ with $n_f \geq n$ and $x^T_i x_j = 0$ for any pair $i, j \in V$ with $i\neq j$. 
Note that the number of possible feature states $n$ must be finite even in an infinitely large graph, as long as the number of features are finite, i.e. $n_f < \infty$.
Since the feature values must be pairwise orthogonal, there can maximally exist $n_f$ distinct feature vectors, as $n_f$ orthogonal vectors would form a basis of $ \mathbb{R}^{n_f}$ and therefore an additional vector would become linearly dependent on the basis vectors.

So let us consider any of the possible feature states $x_f$.
Then $x^T_f h(X) = \sum_{x \in X} x^T_f x = \sum_{x \in X, x = x_f} { 1 = n_{x_f}}$ counts the number of nodes in the set $S$ that have features $x_f$. 
Since this holds for all possible feature vectors $x_f$, all information about any multiset $X$ is preserved by $h(X)$.

{ Accordingly, we can write any multiset function $f(X) = (f(n_{x_1}), \cdots, f(n_{x_n}))$ (which transforms the feature counts) into a function $g$ that extracts first the count information from the sum $h(X)$. Concretely, we can define: $g(h(X))_{f} := f(n_{x_f}) = f(x^T_f h(X))$.}
\end{proof}

While the theorem shares a similarity with proofs of injectivity used in prior work to establish expressive power for learned aggregation, most notably GIN \citep{xu2018how}, our result is a strict generalization, as GIN is restricted to discrete features and one-hot encodings. Crucially, the injective aggregator from their proof is also task-independent and can therefore be treated as fixed, even if this is not considered for their architectural development. We instead question whether these injectivity arguments actually explain the empirical success of GNNs.

\subsection{Loss of Information Over Second Hops}\label{app:second-hops}
We now explore an example of a computational tree of a node with two rounds of sum aggregation, and the qualitative kind of information that is lost from the first to the second hop. As shown by \citet{xu2018how} and generalized in Theorem~\ref{thm:onehop}, sum is injective over one-hot encoded features, but the second aggregation round sums features that are not necessarily orthogonal, and therefore loses neighborhood information. The computational tree and calculation of hops are displayed in Figure~\ref{fig:tree-example}.

\begin{figure*}[h!]
    \centering
    \begin{tikzpicture}
    \node[inner sep=3pt,draw] at (0:0) (1) {1};
    \node[inner sep=2pt,draw,circle] at (40:1) (2) {2};
    \node[inner sep=3pt,draw] at (40:2) (4) {4};
    \node[inner sep=2pt,draw,circle] at (20:2) (5) {5};

    \node[inner sep=2pt,draw,circle] at (-45:1) (3) {3};
    \node[inner sep=3pt,draw] at (-10:2) (6) {6};
    \node[inner sep=3pt,draw] at (-30:2) (7) {7};
    \node[inner sep=2pt,draw,circle] at (-50:2) (8) {8};

    \draw (1) -- (2) -- (4);
    \draw (2) -- (5);
    
    \draw (1) -- (3) -- (6);
    \draw (3) -- (7);
    \draw (3) -- (8);

    \begin{scope}[xshift=7cm]
    \node at (0,0) {\begin{tabular}{l}
       Feature encodings:
       $\square = {\binom{1}{0}} $  and $\bigcirc = {\binom{0}{1}} $  \\\\
       $H_1^{(0)}={\binom{1}{0}}$  \\[4pt]
       $H_2^{(0)}=H_3^{(0)}={\binom{0}{1}} \implies H_1^{(1)} = {\binom{0}{2}}$\\[4pt]
       $H_2^{(1)}={\binom{2}{1}}$ and $H_3^{(1)}={\binom{3}{1}} \implies H_1^{(2)} = {\binom{5}{2}}$
    \end{tabular}};
    \end{scope}
    \end{tikzpicture}
    \caption{Example of a two-hop neighborhood with one-hot encoded features and sum aggregation.}
    \label{fig:tree-example}
\end{figure*}
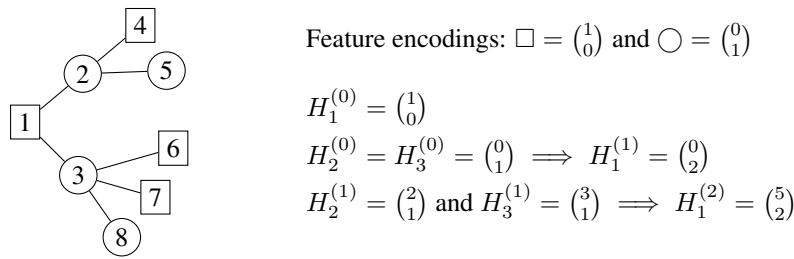

Given the previous hops $H_1^{(0)}={\binom{1}{0}}$ and $H_1^{(1)}={\binom{0}{2}}$, and the second hop $H_1^{(2)}={\binom{5}{2}}$, what other combinations of two-hop neighborhoods can there be for node 1? Apart from itself, node 1's two-hop neighbors are in a ${\binom{3}{2}}$ feature ratio. However, we have lost the ability to recognize a) how many belong to each of its one-hop neighbor; and b) the distribution or homogeneity of each neighborhood. In reality, these 5 nodes are approximately spread out in number and distribution across the one-hop neighbors 2 and 3. But alternatively, all 5 nodes could have belonged to node 2, or all squared nodes could have belonged to node 3.

Note that, without previous hops $H_1^{(0)}$ and $H_1^{(1)}$, we cannot even distinguish node 1's original features, nor distinguish its presence as a neighbor in its one-hop neighbors. Therefore, concatenating all hops is advantageous.

For completeness, we also include the calculations for mean aggregation, which are qualitatively similar to the sum in this case. 

\begin{tabular}{ll}
       $H_1^{(0)}={\binom{1}{0}}$  ;  \quad
       $H_2^{(0)}=H_3^{(0)}={\binom{0}{1}}$ & $\implies H_1^{(1,m)} = {\binom{0}{1}}$ \\[4pt]
       $H_2^{(1,m)}={\binom{2/3}{1/3}}$ and $H_3^{(1.m)}={\binom{3/4}{1/4}}$ & $\implies H_1^{(2,m)} = {\binom{17/24}{7/24}}$
\end{tabular}

A good lossless aggregation scheme should take all possible second-hop neighborhood distributions and map them to values that would not lose information when aggregated. For instance, choosing $a,b$ such that ${\binom{2}{1}}\rightarrow a$ from node 2, and ${\binom{3}{1}}\rightarrow b$ from node 3, so that, $H_1^{(2)}=a+b$ could recover both values separately. Naturally, mapping them to one-hot encodings per distribution would suffice, but it would grow exponentially. This opens the door for better suitable fixed aggregations, or perhaps other kinds of learnable aggregation beyond current understanding of message passing.


\section{Experimental Details}
The codebase can be found here: \url{https://github.com/celrm/fixed-aggregation-features}.

\label{app:experimentalsetup}

\subsection{Dataset Details}
Datasets are taken directly from the setup of \citet{luo2024classic}, which includes varied node classification datasets. Here in Table~\ref{tab:datasets} we include for completeness the same overview of these benchmarks.

\begin{table*}[h!]
\centering
\caption{Details of the node classification datasets.}
\label{tab:datasets}
\resizebox{\linewidth}{!}{
\begin{tabular}{l l r r r r l l}
\toprule
{Dataset} & {Type} & {\# Nodes} & {\# Edges} & {\# Features} & {Classes} & {Metric} & {Origin} \\
\midrule
Cora            & Homophily   & 2,708   & 5,278   & 1,433  & 7  & Accuracy & \citep{Cora} \\
CiteSeer        & Homophily   & 3,327   & 4,522   & 3,703  & 6  & Accuracy & \citep{Citeseer} \\
PubMed          & Homophily   & 19,717  & 44,324  & 500    & 3  & Accuracy & \citep{Pubmed} \\
Computer        & Homophily   & 13,752  & 245,861 & 767    & 10 & Accuracy & \citep{shchur2018pitfalls} \\
Photo           & Homophily   & 7,650   & 119,081 & 745    & 8  & Accuracy & \citep{shchur2018pitfalls} \\
CS              & Homophily   & 18,333  & 81,894  & 6,805  & 15 & Accuracy & \citep{shchur2018pitfalls} \\
Physics         & Homophily   & 34,493  & 247,962 & 8,415  & 5  & Accuracy & \citep{shchur2018pitfalls} \\
WikiCS          & Homophily   & 11,701  & 216,123 & 300    & 10 & Accuracy & \citep{wikics} \\
Squirrel        & Heterophily & 2,223   & 46,998  & 2,089  & 5  & Accuracy & \citep{heterophilous} \\
Chameleon       & Heterophily & 890     & 8,854   & 2,325  & 5  & Accuracy & \citep{heterophilous} \\
Roman-Empire    & Heterophily & 22,662  & 32,927  & 300    & 18 & Accuracy & \citep{platonov2023critical} \\
Amazon-Ratings  & Heterophily & 24,492  & 93,050  & 300    & 5  & Accuracy & \citep{platonov2023critical} \\
Minesweeper     & Heterophily & 10,000  & 39,402  & 7      & 2  & ROC-AUC  & \citep{platonov2023critical} \\
Questions       & Heterophily & 48,921  & 153,540 & 301    & 2  & ROC-AUC  & \citep{platonov2023critical} \\
\bottomrule
\end{tabular}
}
\end{table*}

\subsection{Hyperparameters}
Each experiment is run on an NVIDIA A100 GPU. The setup is taken from \citet{luo2024classic}. That is, for a maximum of 2500 epochs, we tune the following parameters:
\begin{enumerate}
    \item DROPOUT $\in$ (0.0 0.2 0.3 0.5 0.7)
    \item LR $\in$ (0.01 0.005 0.001 0.0001)
    \item NORMALIZATION $\in$ (ln bn none)
    \item HIDDEN CHANNELS $\in$ (64 256 512)
\end{enumerate}

While \citet{luo2024classic} includes weight decay as a hyperparameter, there are no concrete ranges specified for it. Therefore, we tune all the different values from the best runs of the given datasets: 
\begin{enumerate}
\item[5. ] WEIGHT DECAY $\in$ (0.0 1e-2 1e-3 5e-4 5e-5). 
\end{enumerate}

Moreover, \citet{luo2024classic} tunes the local layers from 1 to 10 or 15. We instead take for each dataset the same value that they have found best for the GNNs, and use it to construct our fixed aggregation features up to that depth. In some cases where there are too many features, we restrict the depth to a smaller value, thus including a strict subset of features instead. This serves as an ad hoc feature selection to reduce overfitting.

We include as hyperparameter the MLP depth. We also include results for MLP = 1 in Table~\ref{tab:logreg}.
\begin{enumerate}
\item[6. ] 
MLP LAYERS $\in$ (2 3 5). 
\end{enumerate}

On the other hand, we do not include linear residual connections, as these are used to bypass the convolutional layers in the classical GNNs. This creates a direct difference on the two datasets that most benefit from this component, Minesweeper and Roman-Empire.

In Table~\ref{tab:hyperparams_per_dataset} we include the best hyperparameter choices for the four models: GCN, GAT, GraphSAGE, and FAF$_4$, the results of which are in Table~\ref{tab:valmain}. We run baselines directly from the setup of \citet{luo2024classic}, and we sweep FAF$_4$ with the same ranges included in their work. Each dataset has a specific number of splits given by their setup (from 3 to 10), which we then average.

\subsection{Runtime and Feature Reduction}
We also include in Table~\ref{tab:runtimes} the training runtime of our algorithms, grouped by the number of reducers: FAF$_i$ with $i=|\mathcal{R}|$\textemdash as MLPs with the same input width will have the same training time. Note that all runs have the same number of epochs (2500), as in the original setup, and all datasets match the number of runs of the setup.

\begin{table*}[h!]
  \centering
  \small
\caption{Empirical training time in seconds of FAF and GNN models, averaged over runs.}
\label{tab:runtimes}
\centering\resizebox{\linewidth}{!}{
  \begin{tabular}{ccccccccccc}
    \toprule
    Dataset
      & {GCN}
      & {GAT}
      & {SAGE}
      & {FAF$_4$}
      & {FAF$_2$}
      & {FAF$_1$}
      & {MLP} \\
    \midrule
    computer
      & 127.30 ± 0.58
      & 29.33 ± 3.22
      & 45.33 ± 0.58
      & 42.67 ± 0.58
      & 25.00 ± 0.00
      & 17.00 ± 0.00
      & 9.33 ± 0.58 \\
    photo
      & 82.00 ± 0.00
      & 26.00 ± 0.00
      & 33.33 ± 0.58
      & 66.67 ± 0.58
      & 36.00 ± 0.00
      & 21.33 ± 0.58
      & 7.33 ± 0.58 \\
    ratings
      & 140.30 ± 0.58
      & 161.30 ± 0.58
      & 330.00 ± 0.00
      & 46.00 ± 0.00
      & 26.00 ± 0.00
      & 17.33 ± 0.58
      & 10.00 ± 0.00 \\
    chameleon
      & 22.30 ± 0.48
      & 15.20 ± 0.42
      & 18.00 ± 0.00
      & 44.50 ± 0.97
      & 22.50 ± 0.53
      & 15.10 ± 0.32
      & 10.00 ± 0.00 \\
    citeseer
      & 16.20 ± 0.45
      & 20.00 ± 0.00
      & 26.00 ± 0.00
      & 62.80 ± 1.79
      & 36.20 ± 0.45
      & 25.80 ± 1.79
      & 13.60 ± 0.55 \\
    coauthor-cs
      & 157.00 ± 0.00
      & 70.33 ± 0.58
      & 301.30 ± 20.50
      & 296.30 ± 1.16
      & 159.00 ± 0.00
      & 91.00 ± 0.00
      & 26.33 ± 0.58 \\
    coauthor-physics
      & 65.33 ± 0.58
      & 190.00 ± 0.00
      & 400.30 ± 0.58
      & 2383.00 ± 0.58
      & 1004.00 ± 1.00
      & 532.70 ± 0.58
      & 183.00 ± 0.00 \\
    cora
      & 16.40 ± 0.89
      & 20.20 ± 0.45
      & 11.20 ± 0.45
      & 46.80 ± 0.45
      & 26.00 ± 0.00
      & 17.00 ± 0.00
      & 8.20 ± 0.45 \\
    minesweeper
      & 69.67 ± 0.58
      & 100.30 ± 0.58
      & 68.67 ± 2.08
      & 19.00 ± 0.00
      & 20.00 ± 3.46
      & 19.00 ± 0.00
      & 18.00 ± 1.73 \\
    pubmed
      & 20.20 ± 0.45
      & 42.00 ± 0.00
      & 78.20 ± 0.45
      & 34.00 ± 0.00
      & 19.20 ± 0.45
      & 12.00 ± 0.00
      & 6.20 ± 0.45 \\
    questions
      & 650.00 ± 0.00
      & 258.00 ± 1.00
      & 363.70 ± 0.58
      & 180.70 ± 0.58
      & 122.70 ± 0.58
      & 94.00 ± 0.00
      & 64.67 ± 1.16 \\
    roman-empire
      & 240.30 ± 0.58
      & 294.30 ± 0.58
      & 93.00 ± 0.00
      & 38.67 ± 2.89
      & 23.33 ± 0.58
      & 18.00 ± 0.00
      & 14.67 ± 2.89 \\
    squirrel
      & 29.00 ± 0.00
      & 87.10 ± 0.32
      & 24.70 ± 1.89
      & 43.30 ± 0.48
      & 25.20 ± 0.42
      & 18.00 ± 0.00
      & 11.70 ± 0.48 \\
    wikics
      & 60.00 ± 3.46
      & 97.33 ± 0.58
      & 27.33 ± 0.58
      & 10.00 ± 0.00
      & 7.33 ± 0.58
      & 7.00 ± 0.00
      & 7.00 ± 0.00 \\
    \bottomrule
  \end{tabular}
}
\end{table*}

In general, MLPs are more efficient than MPGNNs, as backpropagation over message-passing is costly. However, we increase the number of features in the data\textemdash depending on the aggregation depth and number of reducers\textemdash so for some datasets with many features the improvement is not necessarily observed. Moreover, this can lead to memory constraints. A way to reduce this overhead is to apply dimensionality reduction to the tabular FAF representation. 

Applying PCA before aggregation discards information from the original node features, which is undesirable. Post-aggregation reduction is more natural because many aggregated columns can be redundant. That said, PCA mixes raw features into dense linear combinations, which can make optimization harder when the task signal is sparse, e.g., concentrated in interactions between a few specific features. As an alternative, in Table~\ref{tab:dropcols} we drop columns that lie in the span of others while leaving the original features untouched. This preserves performance, and the runtime gain is approximately proportional to the number of dropped features.

\begin{table*}[h!]
  \centering
  \small
\caption{Comparison of FAF$_4$ against post-dropping columns in the same span than previous ones, and to PCA.}
\label{tab:dropcols}
\centering\resizebox{\linewidth}{!}{
\begin{tabular}{l*{15}{c}}
\toprule
Dataset
& \multicolumn{3}{c}{Cora}
& \multicolumn{3}{c}{Citeseer}
& \multicolumn{3}{c}{Pubmed}
& \multicolumn{3}{c}{Coauthor-CS}
& \multicolumn{3}{c}{Amazon-Photo} \\
\cmidrule(lr){2-4}
\cmidrule(lr){5-7}
\cmidrule(lr){8-10}
\cmidrule(lr){11-13}
\cmidrule(lr){14-16}
Model
& FAF$_4$ & Drop-Span & PCA
& FAF$_4$ & Drop-Span & PCA
& FAF$_4$ & Drop-Span & PCA
& FAF$_4$ & Drop-Span & PCA
& FAF$_4$ & Drop-Span & PCA \\
\midrule
Test perf.
& 0.818 & 0.8226 & 0.6632
& 0.6768 & 0.6846 & 0.471
& 0.772 & 0.7688 & 0.627
& 0.9541 & 0.95 & 0.9203
& 0.9651 & 0.9612 & 0.9388 \\
OG. dim.
& 1433 & 1433 & 1433
& 3703 & 3703 & 3703
& 500 & 500 & 500
& 6805 & 6805 & 6805
& 745 & 745 & 745 \\
MLP dim.
& 18629 & 2708 & 2708
& 33327 & 3327 & 3327
& 8500 & 7364 & 8500
& 88465 & 18333 & 18333
& 18625 & 7650 & 7650 \\
\bottomrule
\end{tabular}
}
\end{table*}

\subsection{FAFs Beyond Transductive Node Classification}

Our theory applies to any task that learns multiset functions over neighborhoods. In our experiments, we focus on node classification for two main reasons. First, these are the benchmarks on which GNNs have been shown to be competitive with more complex architectures in \citet{luo2024classic}, so they are amenable to simple models for which we have strong, well-tuned hyperparameters. Second, node-classification datasets typically provide rich features that depend on neighborhood distributions. Thus, non-injective but commonly used reducers such as mean and sum still convey highly informative distributional signals. Regarding inductive settings, they would require computing the new aggregation rounds at test time. We would not have access to the test node features when precomputing training aggregations. Otherwise, our approach is just as feasible as in the transductive case.

\begin{table*}[h!]
\caption{Best hyperparameters for FAF. Classic GNNs are taken from \citep{luo2024classic}.}
\label{tab:hyperparams_per_dataset}
\centering
\small
\begin{tabular}{llccccccccc}
\toprule
Dataset & Model & dropout & lr & bn & ln & hidden channels & wd & hops & mlp layers & res \\
\midrule
\multirow[t]{4}{*}{computer} & GCN & 0.5 & 0.001 & 0 & 1 & 512 & 5e-05 & 3 & 0 & 0 \\
 & GAT & 0.5 & 0.001 & 0 & 1 & 64 & 5e-05 & 2 & 0 & 0 \\
 & SAGE & 0.3 & 0.001 & 0 & 1 & 64 & 5e-05 & 4 & 0 & 0 \\
 & FAF & 0.7 & 0.005 & 1 & 0 & 256 & 5e-05 & 2 & 2 & 0 \\
\cline{1-11}
\multirow[t]{4}{*}{photo} & GCN & 0.5 & 0.001 & 0 & 1 & 256 & 5e-05 & 6 & 0 & 1 \\
 & GAT & 0.5 & 0.001 & 0 & 1 & 64 & 5e-05 & 3 & 0 & 1 \\
 & SAGE & 0.2 & 0.001 & 0 & 1 & 64 & 5e-05 & 6 & 0 & 1 \\
 & FAF & 0.5 & 0.005 & 1 & 0 & 256 & 0.0005 & 4 & 2 & 0 \\
\cline{1-11}
\multirow[t]{4}{*}{ratings} & GCN & 0.5 & 0.001 & 1 & 0 & 512 & 0 & 4 & 0 & 1 \\
 & GAT & 0.5 & 0.001 & 1 & 0 & 512 & 0 & 4 & 0 & 1 \\
 & SAGE & 0.5 & 0.001 & 1 & 0 & 512 & 0 & 9 & 0 & 1 \\
 & FAF & 0.2 & 0.001 & 1 & 0 & 256 & 0 & 3 & 2 & 0 \\
\cline{1-11}
\multirow[t]{4}{*}{chameleon} & GCN & 0.2 & 0.005 & 0 & 0 & 512 & 0.001 & 5 & 0 & 0 \\
 & GAT & 0.7 & 0.01 & 1 & 0 & 256 & 0.001 & 2 & 0 & 1 \\
 & SAGE & 0.7 & 0.01 & 1 & 0 & 256 & 0.001 & 4 & 0 & 1 \\
 & FAF & 0.3 & 0.001 & 1 & 0 & 512 & 0.01 & 5 & 5 & 0 \\
\cline{1-11}
\multirow[t]{4}{*}{citeseer} & GCN & 0.5 & 0.001 & 0 & 0 & 512 & 0.01 & 2 & 0 & 0 \\
 & GAT & 0.5 & 0.001 & 0 & 0 & 256 & 0.01 & 3 & 0 & 1 \\
 & SAGE & 0.2 & 0.001 & 0 & 0 & 512 & 0.01 & 3 & 0 & 0 \\
 & FAF & 0 & 0.005 & 0 & 1 & 512 & 0.001 & 2 & 3 & 0 \\
\cline{1-11}
\multirow[t]{4}{*}{coauthor-cs} & GCN & 0.3 & 0.001 & 0 & 1 & 512 & 0.0005 & 2 & 0 & 1 \\
 & GAT & 0.3 & 0.001 & 0 & 1 & 256 & 0.0005 & 1 & 0 & 1 \\
 & SAGE & 0.5 & 0.001 & 0 & 1 & 512 & 0.0005 & 2 & 0 & 1 \\
 & FAF & 0.2 & 0.005 & 1 & 0 & 64 & 0.01 & 2 & 2 & 0 \\
\cline{1-11}
\multirow[t]{3}{*}{coauthor-physics} & GCN & 0.3 & 0.001 & 0 & 1 & 64 & 0.0005 & 2 & 0 & 1 \\
 & GAT & 0.7 & 0.001 & 1 & 0 & 256 & 0.0005 & 2 & 0 & 1 \\
 & SAGE & 0.7 & 0.001 & 1 & 0 & 64 & 0.0005 & 2 & 0 & 1 \\
   & FAF & 0 & 0.001 & 1 & 0 & 512 & 0.001 & 1 & 2 & 0 \\

\cline{1-11}
\multirow[t]{4}{*}{cora} & GCN & 0.7 & 0.001 & 0 & 0 & 512 & 0.0005 & 3 & 0 & 0 \\
 & GAT & 0.2 & 0.001 & 0 & 0 & 512 & 0.0005 & 3 & 0 & 1 \\
 & SAGE & 0.7 & 0.001 & 0 & 0 & 256 & 0.0005 & 3 & 0 & 0 \\
 & FAF & 0.7 & 0.01 & 0 & 1 & 512 & 0.01 & 3 & 3 & 0 \\
\cline{1-11}
\multirow[t]{4}{*}{minesweeper} & GCN & 0.2 & 0.01 & 1 & 0 & 64 & 0 & 12 & 0 & 1 \\
 & GAT & 0.2 & 0.01 & 1 & 0 & 64 & 0 & 15 & 0 & 1 \\
 & SAGE & 0.2 & 0.01 & 1 & 0 & 64 & 0 & 15 & 0 & 1 \\
 & FAF & 0.2 & 0.01 & 1 & 0 & 64 & 0 & 4 & 12 & 0 \\
\cline{1-11}
\multirow[t]{4}{*}{pubmed} & GCN & 0.7 & 0.005 & 0 & 0 & 256 & 0.0005 & 2 & 0 & 0 \\
 & GAT & 0.5 & 0.01 & 0 & 0 & 512 & 0.0005 & 2 & 0 & 0 \\
 & SAGE & 0.7 & 0.005 & 0 & 0 & 512 & 0.0005 & 4 & 0 & 0 \\
 & FAF & 0.7 & 0.01 & 0 & 1 & 64 & 0 & 4 & 2 & 0 \\
\cline{1-11}
\multirow[t]{4}{*}{questions} & GCN & 0.3 & 3e-05 & 0 & 0 & 512 & 0 & 10 & 0 & 1 \\
 & GAT & 0.2 & 3e-05 & 0 & 1 & 512 & 0 & 3 & 0 & 1 \\
 & SAGE & 0.2 & 3e-05 & 0 & 1 & 512 & 0 & 6 & 0 & 0 \\
 & FAF & 0.2 & 0.005 & 1 & 0 & 512 & 0.01 & 4 & 3 & 0 \\
\cline{1-11}
\multirow[t]{4}{*}{roman-empire} & GCN & 0.5 & 0.001 & 1 & 0 & 512 & 0 & 9 & 0 & 1 \\
 & GAT & 0.3 & 0.001 & 1 & 0 & 512 & 0 & 10 & 0 & 1 \\
 & SAGE & 0.3 & 0.001 & 1 & 0 & 256 & 0 & 9 & 0 & 0 \\
 & FAF & 0.7 & 0.01 & 1 & 0 & 256 & 0 & 2 & 3 & 0 \\
\cline{1-11}
\multirow[t]{4}{*}{squirrel} & GCN & 0.7 & 0.01 & 1 & 0 & 256 & 0.0005 & 4 & 0 & 1 \\
 & GAT & 0.5 & 0.005 & 1 & 0 & 512 & 0.0005 & 7 & 0 & 1 \\
 & SAGE & 0.7 & 0.01 & 1 & 0 & 256 & 0.0005 & 3 & 0 & 1 \\
 & FAF & 0.7 & 0.01 & 1 & 0 & 512 & 0.01 & 4 & 5 & 0 \\
\cline{1-11}
\multirow[t]{4}{*}{wikics} & GCN & 0.5 & 0.001 & 0 & 1 & 256 & 0 & 3 & 0 & 0 \\
 & GAT & 0.7 & 0.001 & 0 & 1 & 512 & 0 & 2 & 0 & 1 \\
 & SAGE & 0.7 & 0.001 & 0 & 1 & 256 & 0 & 2 & 0 & 0 \\
 & FAF & 0.7 & 0.01 & 1 & 0 & 64 & 0.001 & 2 & 2 & 0 \\
\cline{1-11}
\bottomrule
\end{tabular}
\end{table*}

\newpage

\section{Performance Details For the Main Experiments} \label{app:mainexpperformance}

\subsection{Validation and Test Accuracy of Main Results}\label{app:maintables}

In Tables \ref{tab:valmain},~\ref{tab:testmain} we report the validation and test accuracy of the main FAF variants of our experimental results (\S~\ref{s:experiments}) where in Table~\ref{tab:testmain2} we only have the best validation FAF's test results. FAFs in all datasets are $\pm 1\%$ away from the best classic GNN, except for those already mentioned in the main text (Citeseer, Cora, Roman-Empire, and Minesweeper).

\begin{table*}[h!]
\caption{Validation accuracy on node classification: all FAFs against classic GNNs. Test accuracy is shown in Table~\ref{tab:testmain}.}
\label{tab:valmain}
\centering\resizebox{\linewidth}{!}{%
\begin{tabular}{lccccccc}
\toprule
Dataset & computer & photo & ratings & chameleon & citeseer & coauthor-cs & coauthor-physics \\
\midrule
GCN & 92.58 ± 0.10 & 95.42 ± 0.11 & 54.01 ± 0.23 & 48.15 ± 2.35 & \textbf{70.36 ± 0.09} & \underline{95.32 ± 0.07} & \textbf{97.16 ± 0.07} \\
GAT & 92.86 ± 0.06 & 95.93 ± 0.15 & 55.56 ± 0.68 & 46.97 ± 2.07 & \underline{69.52 ± 0.27} & 95.30 ± 0.08 & \underline{97.11 ± 0.03} \\
SAGE & 92.33 ± 0.17 & 95.60 ± 0.16 & \textbf{55.90 ± 0.54} & 46.22 ± 2.12 & 68.48 ± 1.05 & \textbf{95.51 ± 0.04} & 97.02 ± 0.09 \\ 
MLP & 87.89 ± 0.13 & 93.33 ± 0.07 & 48.98 ± 0.72 & 41.43 ± 1.77 & 56.80 ± 1.24 & 93.70 ± 0.07 & 95.89 ± 0.02 \\
\midrule
FAF$_4$ & \underline{93.05 ± 0.04} & \textbf{96.34 ± 0.07} & 55.53 ± 0.43 & \textbf{48.51 ± 2.31} & 67.28 ± 0.64 & 94.93 ± 0.07 & 96.83 ± 0.01 \\
FAF$_{\rm {mean},\rm {std}}$ & 93.04 ± 0.13 & \underline{96.23 ± 0.08} & 55.11 ± 0.40 & \underline{48.42 ± 1.64} & 67.20 ± 0.28 & 94.94 ± 0.07 & 96.84 ± 0.03 \\
FAF$_{\rm {mean}}$ & \textbf{93.16 ± 0.04} & 96.06 ± 0.10 & 53.78 ± 0.52 & 47.99 ± 2.02 & 66.92 ± 0.87 & 95.20 ± 0.14 & 97.00 ± 0.04 \\
FAF$_{\rm max,std}$ & 92.32 ± 0.08 & 95.80 ± 0.04 & \underline{55.70 ± 0.45} & 48.42 ± 2.14 & 66.64 ± 0.54 & 95.04 ± 0.04 & 96.56 ± 0.03 \\
FAF$_{\rm max}$ & 91.93 ± 0.04 & 95.60 ± 0.04 & {55.63 ± 0.29} & 48.06 ± 2.30 & 66.56 ± 0.50 & 95.19 ± 0.13 & 96.54 ± 0.01 \\
FAF$_{\rm {sum}}$ & 90.95 ± 0.04 & 94.88 ± 0.04 & 53.48 ± 0.59 & 47.29 ± 1.92 & 67.84 ± 1.45 & 95.13 ± 0.09 & 96.65 ± 0.05 \\
FAF$_{\rm {std}}$ & 92.50 ± 0.06 & 95.86 ± 0.10 & 55.31 ± 0.32 & 47.27 ± 2.15 & 63.44 ± 0.17 & 95.01 ± 0.12 & 96.76 ± 0.02 \\
\bottomrule
\end{tabular}%
}
\medskip 

\centering\resizebox{\linewidth}{!}{%
\begin{tabular}{lccccccc}
\toprule
Dataset & cora & minesweeper & pubmed & questions & roman-empire & squirrel & wikics \\
\midrule
GCN & 81.28 ± 0.33 & \underline{97.36 ± 0.46} & 79.08 ± 0.23 & 78.63 ± 0.23 & \textbf{91.14 ± 0.58} & 44.88 ± 1.27 & 81.52 ± 0.37 \\
GAT & 81.16 ± 0.52 & 97.08 ± 1.16 & 78.84 ± 0.52 & 78.12 ± 1.03 & \underline{90.49 ± 0.68} & 43.30 ± 1.43 & \textbf{82.38 ± 0.57} \\
SAGE & 81.32 ± 0.41 & \textbf{97.68 ± 0.63} & 78.88 ± 0.91 & 77.35 ± 1.09 & 90.44 ± 0.66 & 40.58 ± 1.17 & \underline{82.27 ± 0.38} \\
MLP & 62.68 ± 1.15 & 51.12 ± 0.93 & 71.12 ± 0.52 & 71.58 ± 1.46 & 66.28 ± 0.27 & 40.57 ± 0.92 & 74.86 ± 0.33 \\
\midrule
FAF$_4$ & 82.84 ± 0.43 & 89.63 ± 1.03 & 79.08 ± 0.36 & \textbf{79.53 ± 1.12} & 78.68 ± 0.19 & \underline{47.31 ± 1.39} & 81.92 ± 0.43 \\
FAF$_{\rm {mean},\rm {std}}$ & \textbf{83.36 ± 0.17} & 89.18 ± 0.71 & \textbf{81.28 ± 0.30} & 77.32 ± 0.36 & 77.59 ± 0.41 & {47.30 ± 1.32} & 81.37 ± 0.51 \\
FAF$_{\rm {mean}}$ & \underline{83.28 ± 0.30} & 89.89 ± 0.93 & \underline{81.16 ± 0.97} & 78.53 ± 0.87 & 76.67 ± 0.36 & 46.29 ± 1.50 & 81.58 ± 0.46 \\
FAF$_{\rm max}$ & 81.80 ± 0.42 & 86.08 ± 0.77 & 77.48 ± 0.30 & \underline{79.15 ± 0.86} & 75.06 ± 0.14 & 46.47 ± 1.38 & 80.30 ± 0.56 \\
FAF$_{\rm max,std}$ & 82.08 ± 0.33 & 87.83 ± 0.63 & 78.28 ± 0.30 & 78.86 ± 0.89 & 76.19 ± 0.26 & \textbf{47.44 ± 1.51} & 80.46 ± 0.53 \\
FAF$_{\rm {sum}}$ & 82.60 ± 0.65 & 89.86 ± 0.85 & 79.40 ± 0.57 & 78.12 ± 0.27 & 77.13 ± 0.23 & 46.85 ± 1.28 & 78.17 ± 0.23 \\
FAF$_{\rm {std}}$ & 81.40 ± 0.51 & 88.20 ± 0.52 & 80.00 ± 0.40 & 76.25 ± 0.53 & 73.95 ± 0.49 & 45.91 ± 1.32 & 77.65 ± 0.31 \\
\bottomrule
\end{tabular}%
}
\end{table*}

\begin{table}
\caption{Test accuracy on node classification: all FAFs against classic GNNs. Validation accuracy is shown in Table~\ref{tab:valmain}.}
\label{tab:testmain}
\centering\resizebox{\linewidth}{!}{
\begin{tabular}{lccccccc}
\toprule
Dataset & computer & photo & ratings & chameleon & citeseer & coauthor-cs & coauthor-physics \\
\midrule
GCN & 93.58 ± 0.44 & 95.77 ± 0.27 & 53.86 ± 0.48 & \underline{44.62 ± 4.50} & \textbf{72.72 ± 0.45} & 95.73 ± 0.15 & \textbf{97.47 ± 0.08} \\
GAT & 93.91 ± 0.22 & 96.45 ± 0.37 & \textbf{55.51 ± 0.55} & 42.90 ± 5.47 & \underline{71.82 ± 0.65} & \underline{96.14 ± 0.08} & \underline{97.12 ± 0.13} \\
SAGE & 93.31 ± 0.17 & 96.17 ± 0.44 & \underline{55.26 ± 0.27} & 43.11 ± 4.73 & \underline{71.82 ± 0.81} & \textbf{96.21 ± 0.10} & 97.10 ± 0.09 \\
MLP & 87.75 ± 0.42 & 93.62 ± 0.36 & 49.04 ± 0.39 & 38.59 ± 3.29 & 57.22 ± 2.25 & 93.80 ± 0.19 & 96.02 ± 0.16 \\
\midrule
FAF$_{\rm best val}$ & \textbf{94.01 ± 0.21} & \underline{96.54 ± 0.13} & 55.09 ± 0.24 & 42.96 ± 2.45 & 70.48 ± 1.24 & 95.37 ± 0.17 & 97.05 ± 0.18 \\
\midrule

FAF$_4$ & 93.75 ± 0.04 & \underline{96.54 ± 0.13} & 54.42 ± 0.45 & 42.96 ± 2.45 & 69.42 ± 1.32 & 95.33 ± 0.20 & 96.96 ± 0.09 \\
FAF$_{\rm mean,std}$ & \underline{94.00 ± 0.25} & 96.30 ± 0.23 & 54.73 ± 0.22 & \textbf{45.13 ± 3.42} & 67.90 ± 0.95 & 95.34 ± 0.14 & 96.93 ± 0.04 \\
FAF$_{\rm mean}$ & \textbf{94.01 ± 0.21} & \textbf{96.71 ± 0.16} & 53.12 ± 0.44 & 43.21 ± 2.24 & 66.82 ± 1.74 & 95.37 ± 0.17 & 97.05 ± 0.18 \\
FAF$_{\rm max,std}$ & 93.60 ± 0.25 & 96.01 ± 0.41 & 55.09 ± 0.24 & 43.20 ± 2.42 & 67.18 ± 0.88 & 95.53 ± 0.10 & 96.61 ± 0.04 \\
FAF$_{\rm max}$ & 92.98 ± 0.22 & 96.12 ± 0.10 & 54.79 ± 0.15 & 42.15 ± 3.19 & 67.52 ± 0.40 & 95.55 ± 0.08 & 96.84 ± 0.13 \\
FAF$_{\rm sum}$ & 91.77 ± 0.24 & 95.08 ± 0.61 & 53.44 ± 0.18 & 39.63 ± 2.90 & 70.48 ± 1.24 & 95.08 ± 0.12 & 96.86 ± 0.06 \\
FAF$_{\rm std}$ & 93.54 ± 0.26 & 96.17 ± 0.10 & 54.77 ± 0.14 & 42.68 ± 2.75 & 62.70 ± 1.18 & 95.77 ± 0.12 & 96.97 ± 0.09 \\
\bottomrule
\end{tabular}
}

\bigskip 

\centering\resizebox{\linewidth}{!}{
\begin{tabular}{lccccccc}
\toprule
Dataset & cora & minesweeper & pubmed & questions & roman-empire & squirrel & wikics \\
\midrule
GCN & \textbf{84.38 ± 0.81} & \underline{97.48 ± 0.06} & \underline{80.00 ± 0.77} & \underline{78.44 ± 0.23} & \textbf{91.05 ± 0.15} & \underline{44.26 ± 1.22} & 80.06 ± 0.81 \\
GAT & 83.02 ± 1.21 & 97.00 ± 1.02 & 79.80 ± 0.94 & 77.72 ± 0.71 & 90.38 ± 0.49 & 39.31 ± 2.42 & \textbf{81.01 ± 0.23} \\
SAGE & \underline{83.18 ± 0.93} & \textbf{97.72 ± 0.70} & 77.42 ± 0.40 & 76.75 ± 1.07 & \underline{90.41 ± 0.10} & 40.22 ± 1.47 & \underline{80.57 ± 0.42} \\
MLP & 58.56 ± 1.75 & 51.74 ± 0.83 & 68.22 ± 0.96 & 70.40 ± 1.17 & 66.43 ± 0.12 & 39.11 ± 1.93 & 72.98 ± 0.49 \\
\midrule
FAF$_{\rm best val}$ & 82.84 ± 0.63 & 90.00 ± 0.39 & \textbf{80.96 ± 1.06} & \textbf{78.69 ± 0.50} & 78.11 ± 0.38 & \textbf{44.59 ± 1.62} & 80.25 ± 0.34 \\
\midrule

FAF$_4$ & 81.44 ± 0.38 & 90.01 ± 0.51 & 77.20 ± 0.45 & \textbf{78.69 ± 0.50} & 78.11 ± 0.38 & 44.02 ± 2.08 & 80.25 ± 0.34 \\
FAF$_{\rm mean,std}$ & 82.84 ± 0.63 & 90.17 ± 0.51 & \textbf{80.96 ± 1.06} & 75.82 ± 1.27 & 77.14 ± 0.52 & 43.83 ± 2.34 & 79.48 ± 0.81 \\
FAF$_{\rm mean}$ & 82.80 ± 0.70 & 90.00 ± 0.39 & 79.88 ± 0.92 & 76.83 ± 1.19 & 76.36 ± 0.55 & 42.44 ± 1.73 & 79.61 ± 0.56 \\
FAF$_{\rm max,std}$ & 79.34 ± 0.95 & 88.36 ± 0.74 & 77.52 ± 0.77 & 76.62 ± 0.79 & 75.89 ± 0.30 & \textbf{44.59 ± 1.62} & 78.44 ± 0.67 \\
FAF$_{\rm max}$ & 79.34 ± 0.67 & 86.39 ± 1.22 & 77.18 ± 0.13 & 77.59 ± 1.67 & 75.01 ± 0.43 & 43.03 ± 1.90 & 78.63 ± 0.35 \\
FAF$_{\rm sum}$ & 81.46 ± 0.62 & 89.96 ± 0.45 & 77.46 ± 0.43 & 76.12 ± 1.08 & 76.90 ± 0.28 & 44.07 ± 1.98 & 76.59 ± 0.36 \\
FAF$_{\rm std}$ & 79.50 ± 0.39 & 88.93 ± 0.68 & 79.06 ± 1.09 & 73.99 ± 1.67 & 73.80 ± 0.21 & 43.63 ± 1.43 & 76.09 ± 0.26 \\
\bottomrule
\end{tabular}
}
\end{table}

\subsection{Training, Validation, and Test Accuracy Curves} \label{app:curves}

In this section we compare training MLPs on FAF features to training GCNs, by tracking train/validation/test accuracy over epochs (Figure~\ref{fig:curves}).
Below we summarize the behaviors on datasets where differences arise between the two methods:

\begin{itemize}[leftmargin=1.5em]
    \item Amazon-Computer (\ref{fig:curve-amazon-computer}) and Amazon-Photo (\ref{fig:curve-amazon-photo}) behave similarly, but GCNs are more unstable.
    \item FAF for Chameleon (\ref{fig:curve-chameleon}) has much better training accuracy but similar generalization; in contrast, GCN for Squirrel (\ref{fig:curve-squirrel}) has much better training, but slightly worse generalization than FAF.
    \item Citeseer (\ref{fig:curve-citeseer}) with FAF breaks at the end of training, which indicates instability. However, this could be overcome with standard learning rate schedules.
    \item Coauthor-CS (\ref{fig:curve-coauthor-cs}) and Coauthor-Physics (\ref{fig:curve-coauthor-physics}) have dips in all metrics for both models.
    \item Questions (\ref{fig:curve-questions}) with FAF is more (locally) unstable but also more stationary and does not degrade performance later on.
    \item As mentioned in the main text, Minesweeper (\ref{fig:curve-minesweeper}) and Roman-Empire (\ref{fig:curve-roman-empire}) are the two datasets that seem to truly lose neighborhood information with FAF.
\end{itemize}


\newcommand{\pair}[2]{
  \begin{subfigure}{\linewidth}
    \centering
    \includegraphics[width=0.4\linewidth]{img/curves/#1-faf.pdf}\hfill
    \includegraphics[width=0.4\linewidth]{img/curves/#1-gcn.pdf}
    \subcaption{#2}\label{fig:curve-#1}
  \end{subfigure}
}

\begin{figure}
\centering
\pair{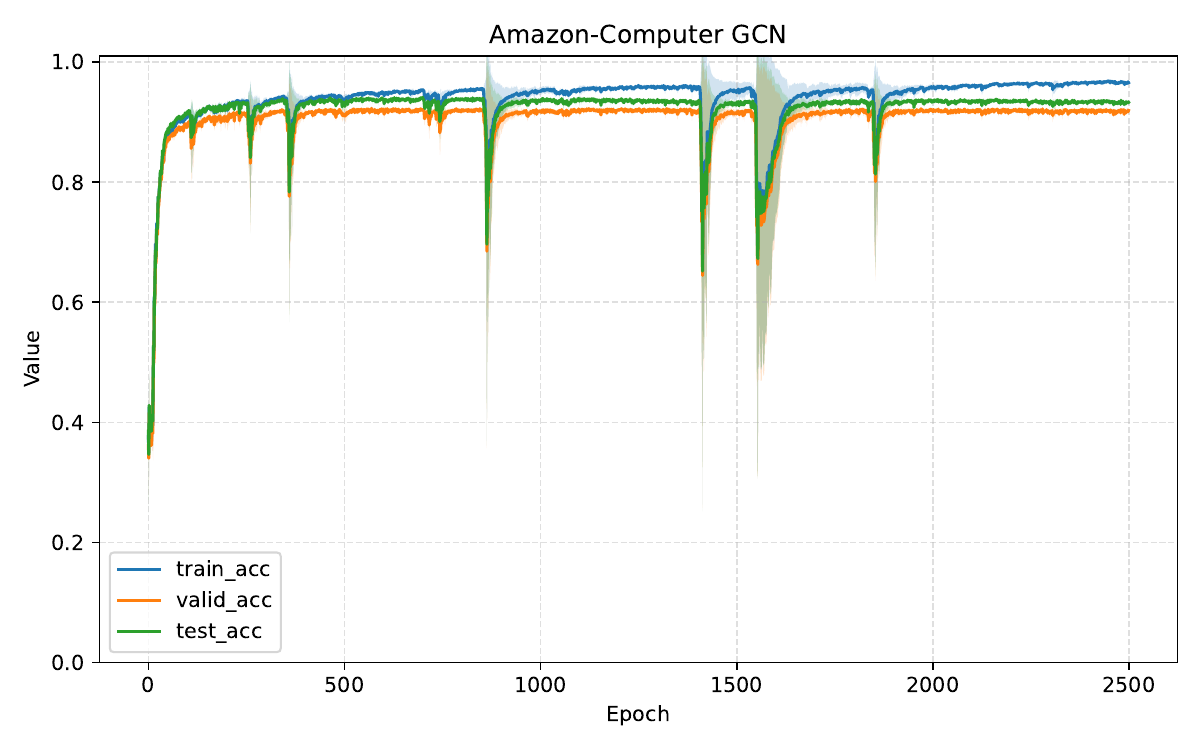}{Amazon-Computer: FAF vs. GCN}

\pair{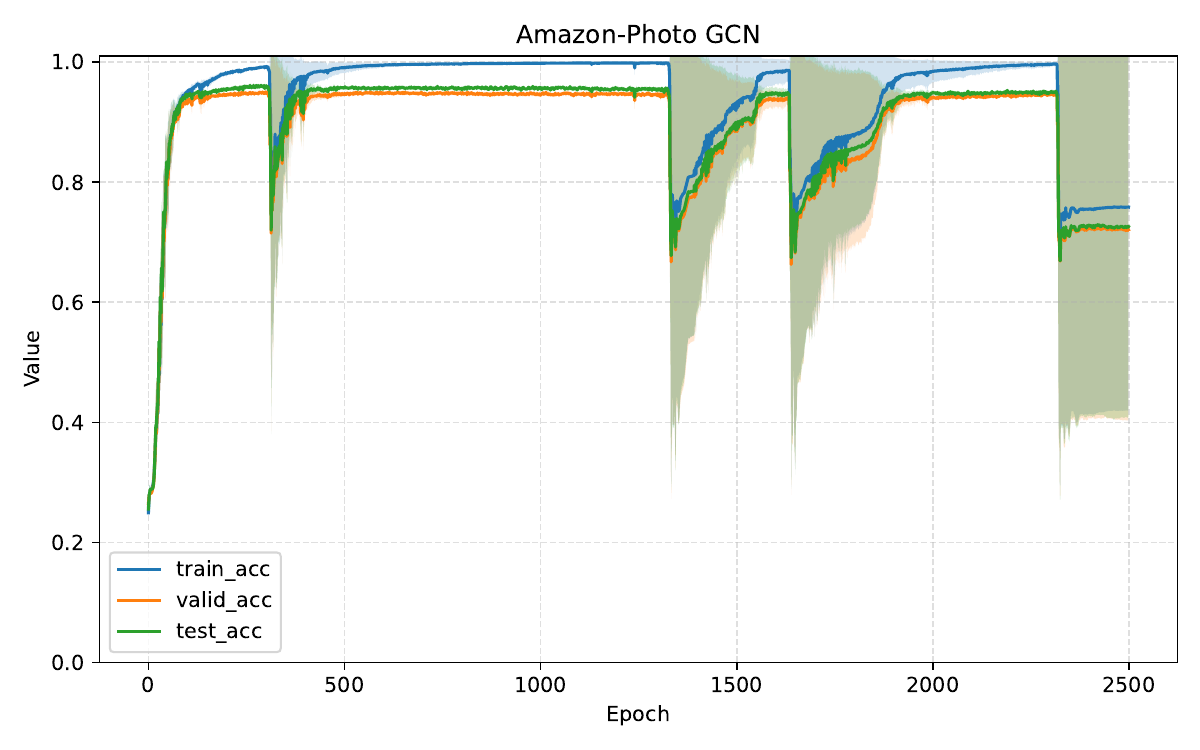}{Amazon-Photo: FAF vs. GCN}

\pair{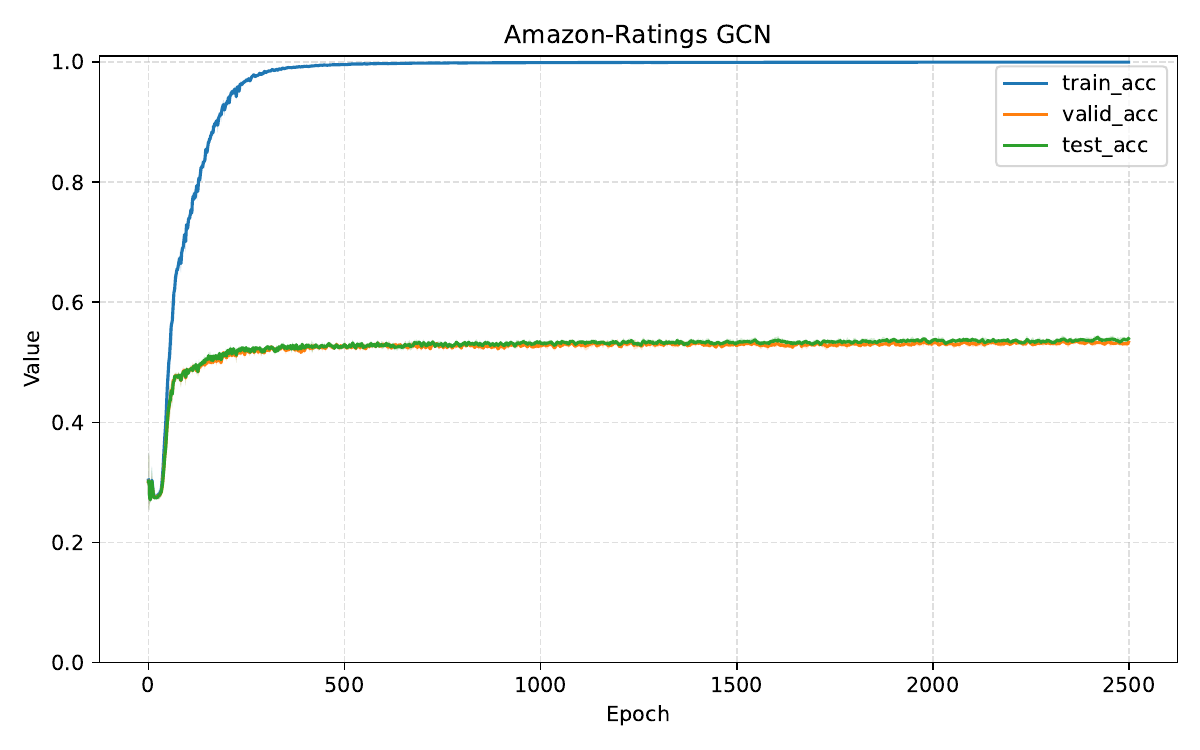}{Amazon-Ratings: FAF vs. GCN}

\pair{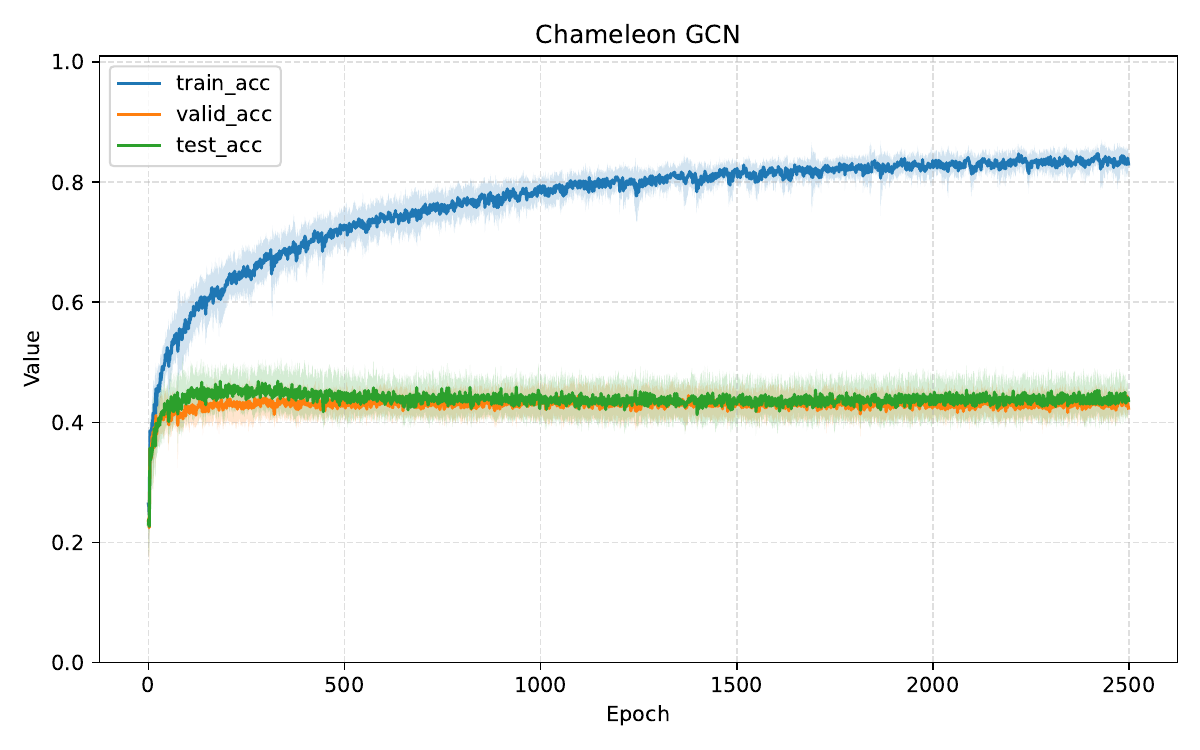}{Chameleon: FAF vs. GCN}

\caption{Train, validation, and test accuracy of FAF+MLP versus GCN. (i)}
\label{fig:curves}
\end{figure}

\begin{figure}
\ContinuedFloat
\centering

\pair{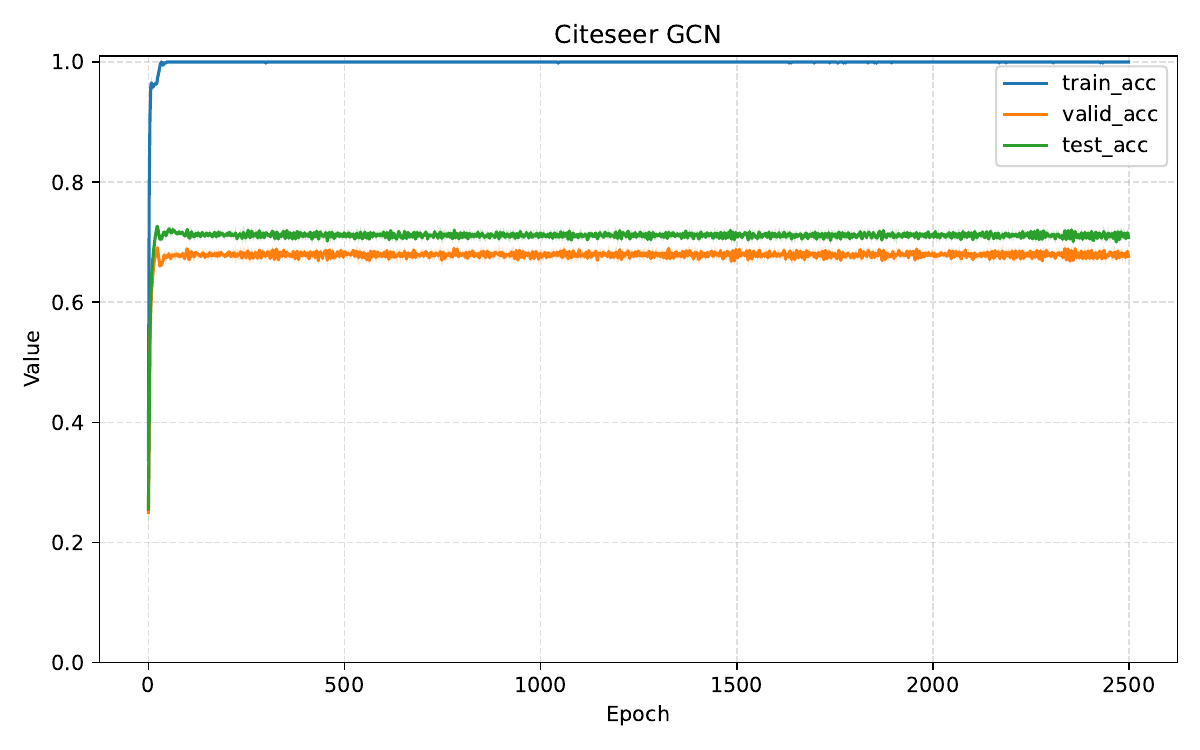}{Citeseer: FAF vs. GCN}

\pair{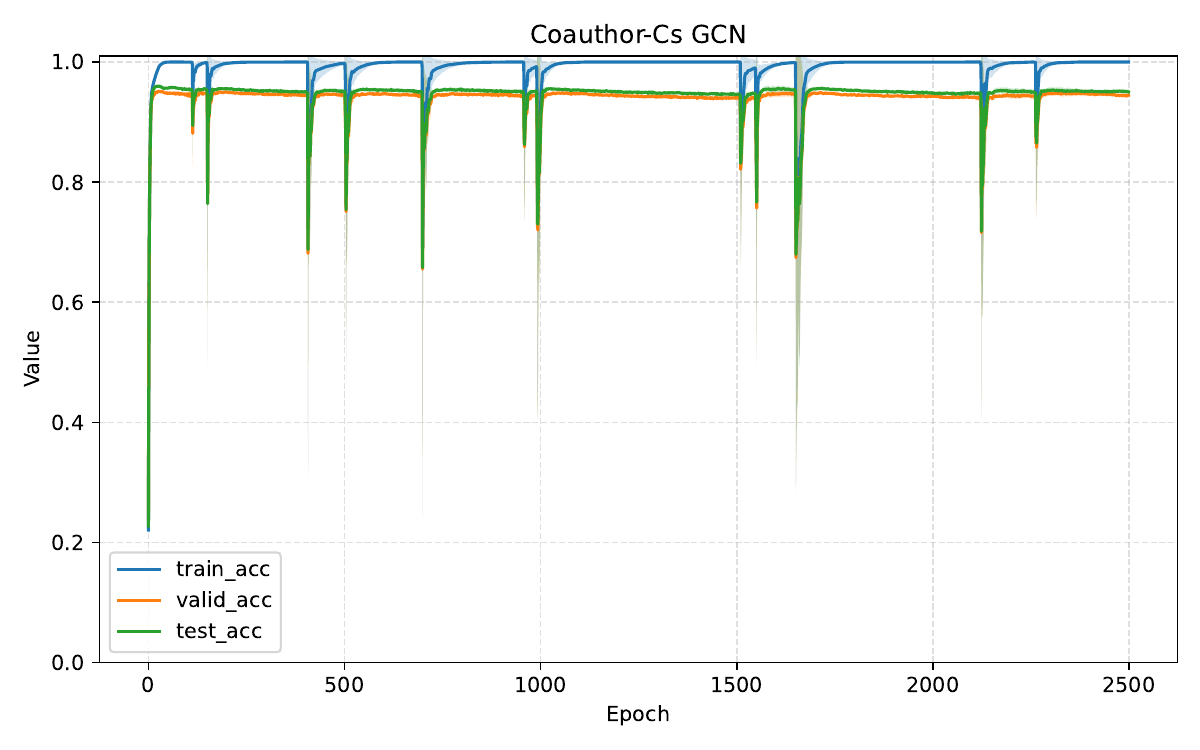}{Coauthor-CS: FAF vs. GCN}

\pair{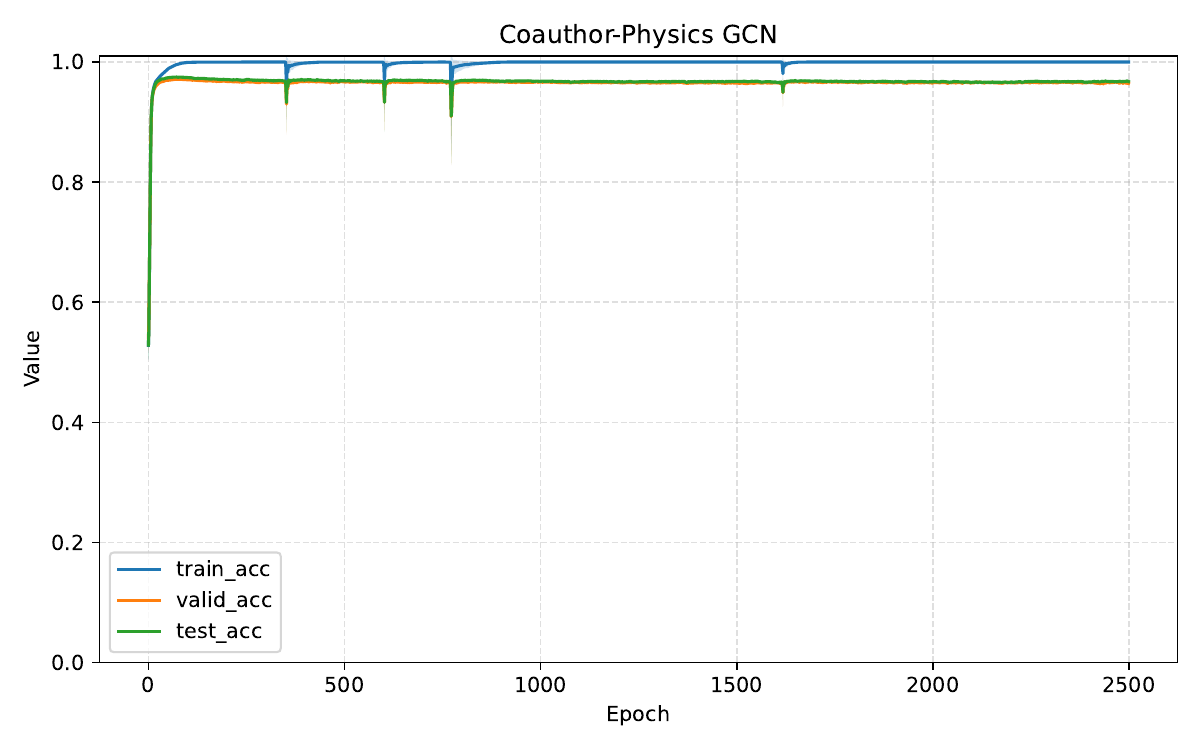}{Coauthor-Physics: FAF vs. GCN}

\pair{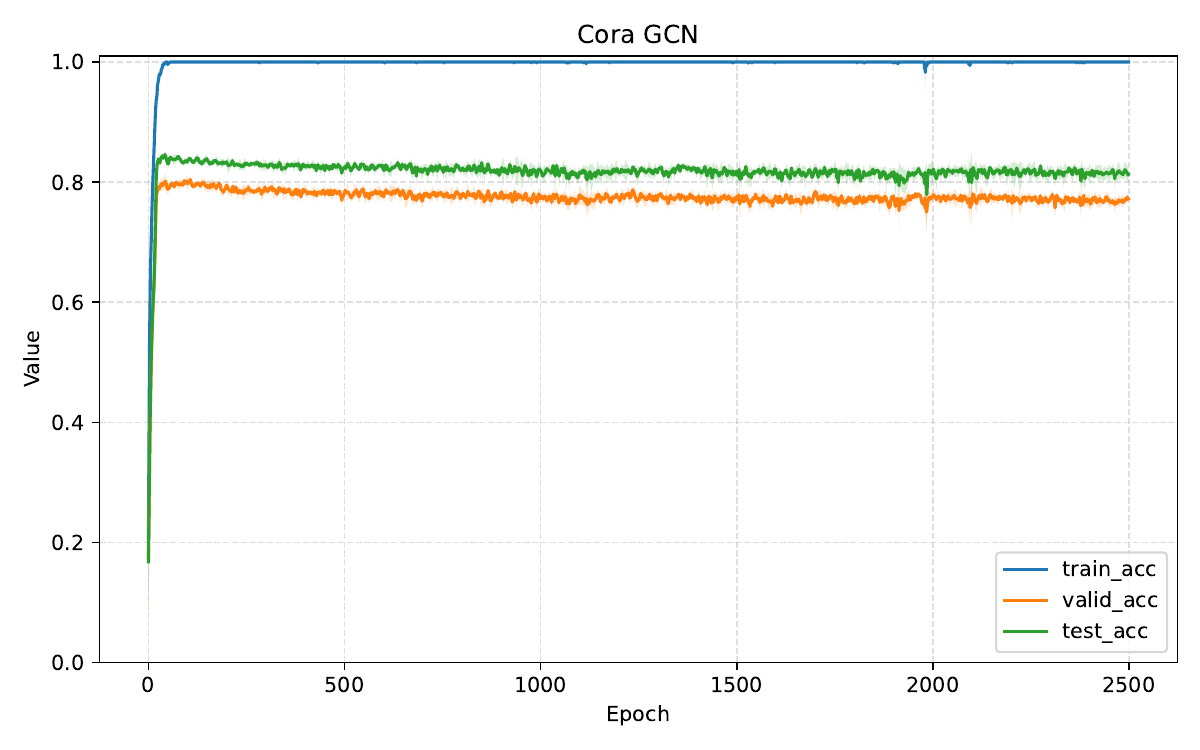}{Cora: FAF vs. GCN}

\caption[]{Train, validation, and test accuracy of FAF+MLP versus GCN. (ii)}

\end{figure}

\begin{figure}
\ContinuedFloat
\centering

\pair{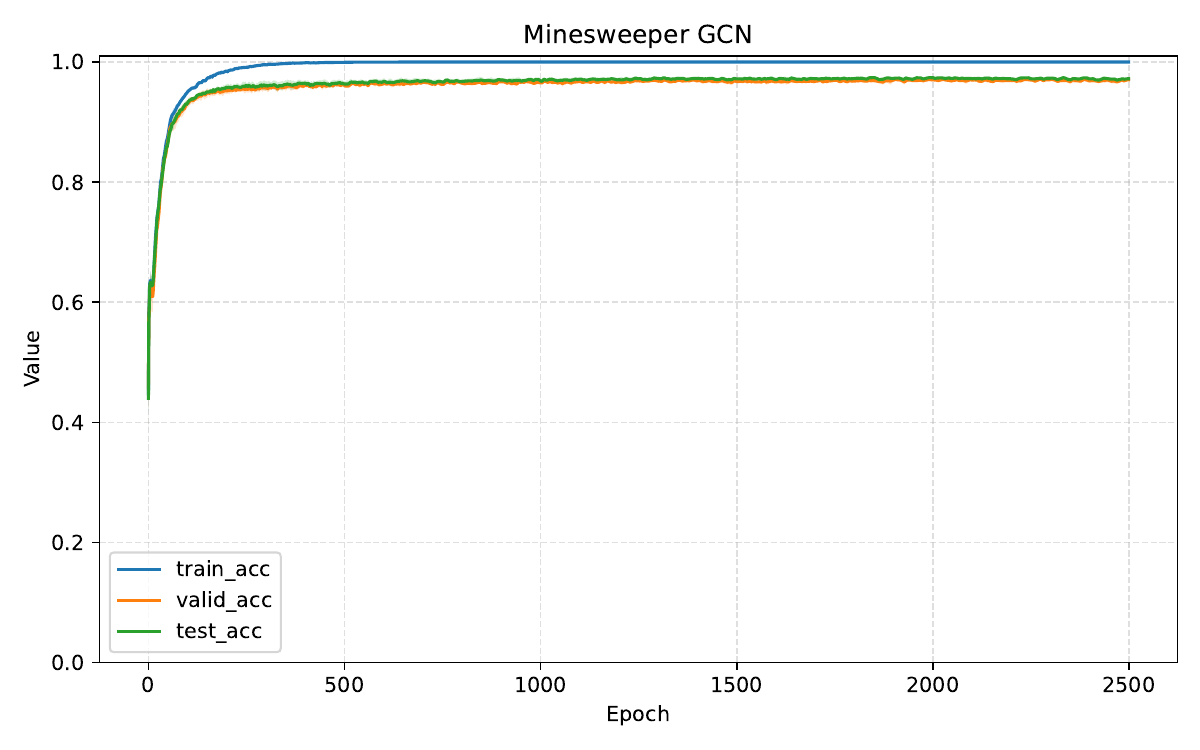}{Minesweeper: FAF vs. GCN}

\pair{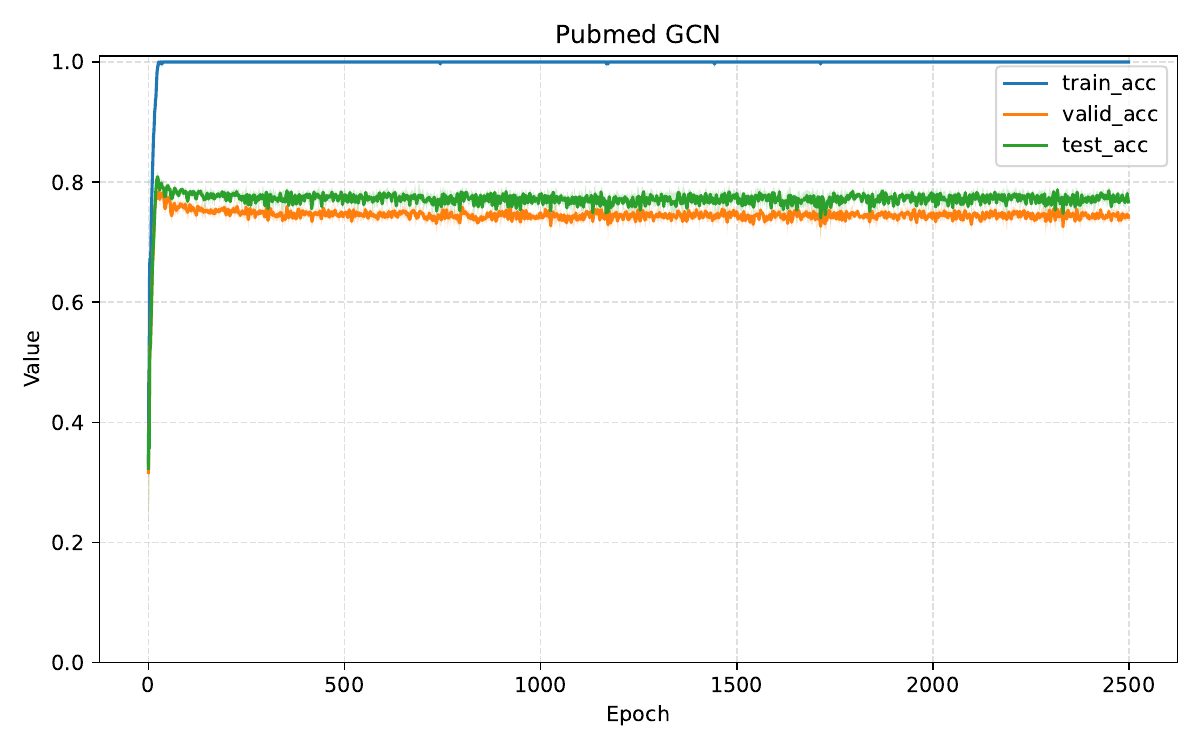}{Pubmed: FAF vs. GCN}

\pair{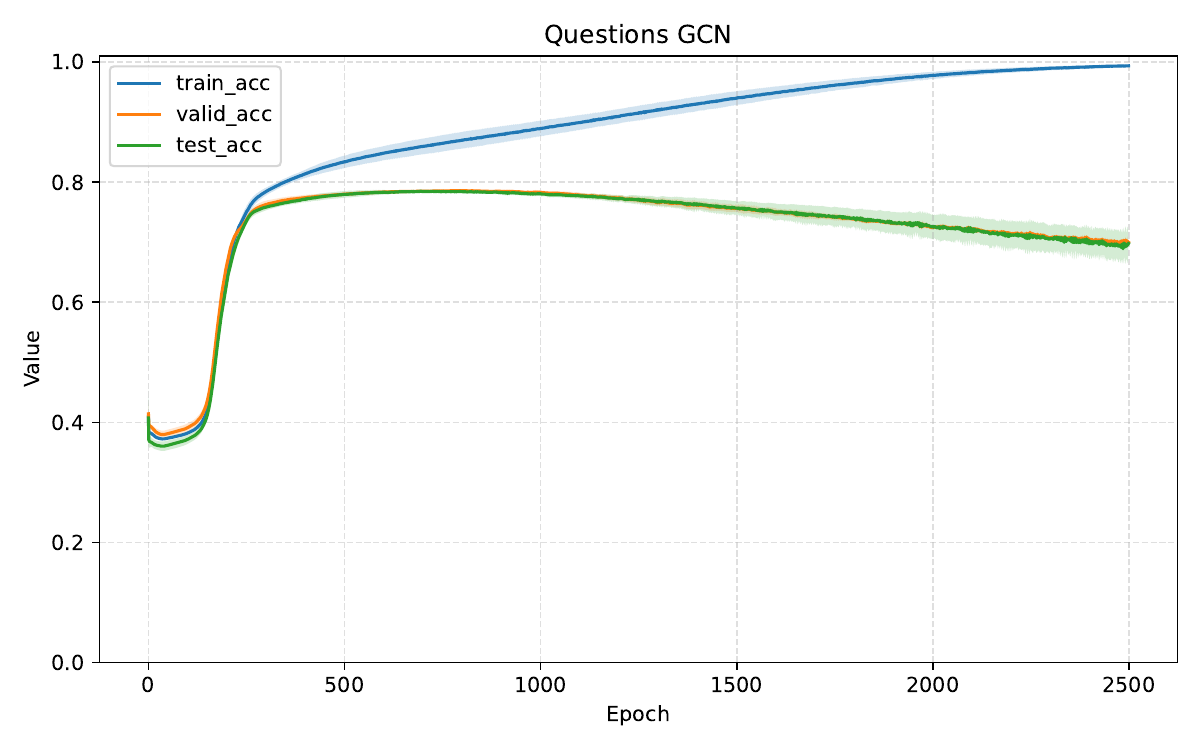}{Questions: FAF vs. GCN}

\pair{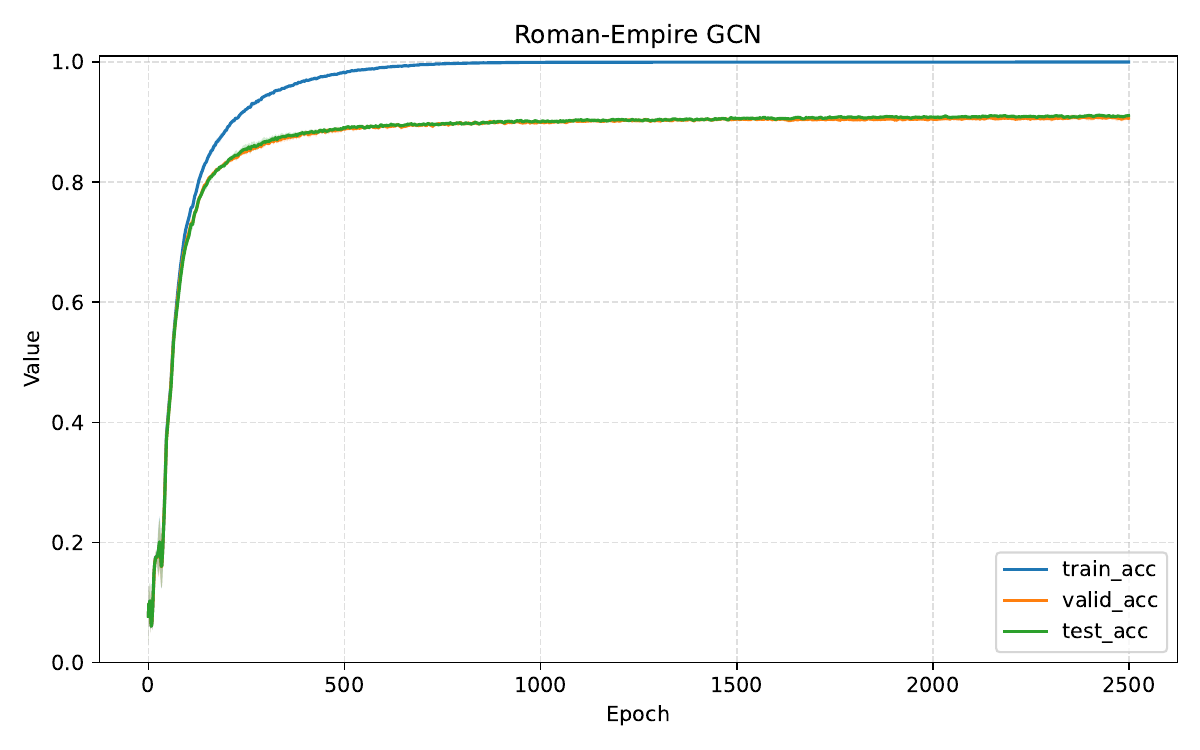}{Roman-Empire: FAF vs. GCN}

\caption[]{Train, validation, and test accuracy of FAF+MLP versus GCN. (iii)}

\end{figure}

\begin{figure}
\ContinuedFloat
\centering

\pair{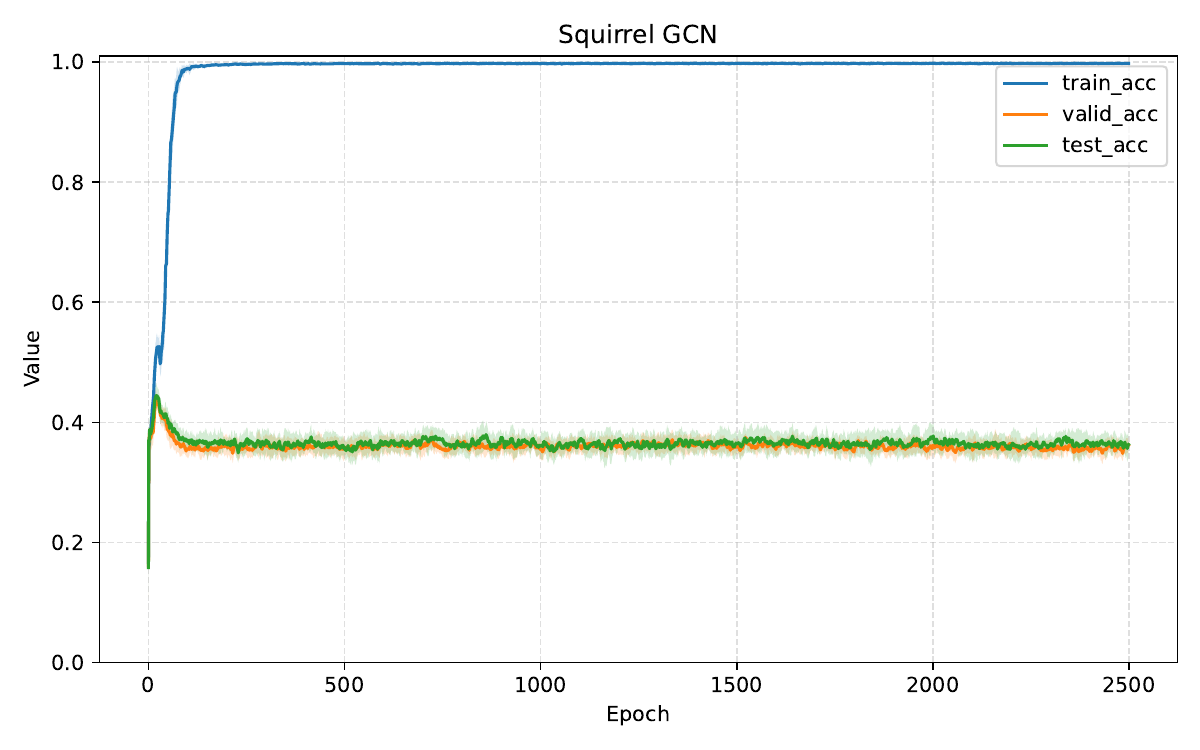}{Squirrel: FAF vs. GCN}

\pair{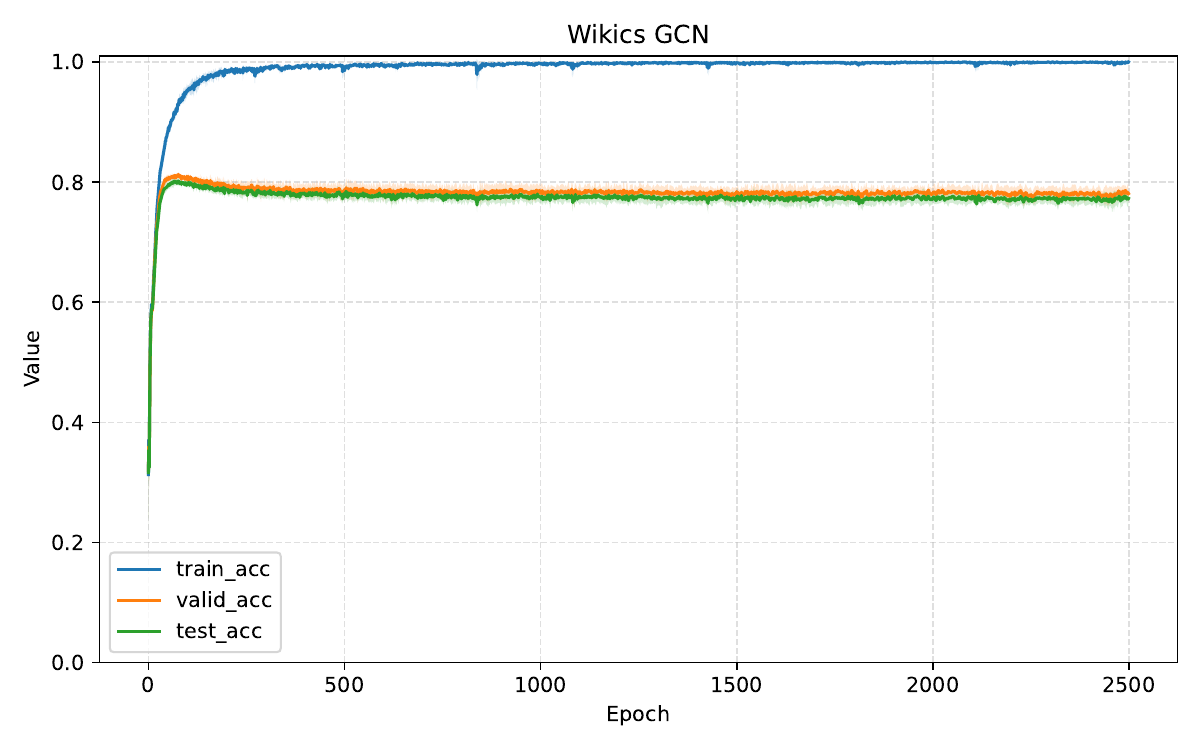}{Wikics: FAF vs. GCN}

\caption[]{Train, validation, and test accuracy of FAF+MLP versus GCN. (iv)}
\end{figure}

\newpage

\section{Additional Experiments and Ablations}\label{app:additional}

\textbf{Number of hops.}\quad
In Tables~\ref{tab:locallayers} and \ref{tab:locallayersval} we include the test and validation for different number of hops concatenated as features for FAF$_4$. One hop already gives much of the information, and two hops often give the best performance.

\begin{table*}[h!]
\caption{Test accuracy for increasing number of concatenated hops in FAF$_4$.}
\label{tab:locallayers}
\centering\resizebox{\linewidth}{!}{
\begin{tabular}{lccccccc}
\toprule
Dataset & computer & photo & ratings & chameleon & citeseer & coauthor-cs & coauthor-physics \\
\midrule
FAF+0 & 87.75 ± 0.42 & 93.62 ± 0.36 & 49.04 ± 0.39 & 38.59 ± 3.29 & 52.98 ± 2.32 & 93.80 ± 0.19 & 96.02 ± 0.16 \\
FAF+1 & 93.41 ± 0.11 & 95.88 ± 0.17 & 53.61 ± 0.07 & \textbf{42.42 ± 3.44} & 65.80 ± 1.15 & 95.11 ± 0.18 & \textbf{96.99 ± 0.11} \\
FAF+2 & 93.76 ± 0.11 & 96.21 ± 0.36 & 53.88 ± 0.20 & 42.31 ± 3.02 & \textbf{69.40 ± 0.56} & \textbf{95.37 ± 0.33} & 96.86 ± 0.13 \\
FAF+4 & 93.93 ± 0.07 & \textbf{96.51 ± 0.14} & 54.67 ± 0.11 & 42.25 ± 2.33 & 51.48 ± 4.42 & 95.31 ± 0.10 & 96.86 ± 0.16 \\
FAF+8 & \textbf{94.12 ± 0.09} & 96.47 ± 0.26 & \textbf{54.81 ± 0.45} & 42.08 ± 1.90 & 42.12 ± 2.96 & 95.17 ± 0.23 & OOM \\
\bottomrule
\end{tabular}
}

\centering\resizebox{\linewidth}{!}{
\begin{tabular}{lccccccc}
\toprule
Dataset & cora & minesweeper & pubmed & questions & roman-empire & squirrel & wikics \\
\midrule
FAF+0 & 58.56 ± 1.75 & 51.74 ± 0.83 & 68.22 ± 0.96 & 70.40 ± 1.17 & 66.43 ± 0.12 & 39.11 ± 1.93 & 72.98 ± 0.49 \\
FAF+1 & 81.12 ± 1.14 & 88.15 ± 0.82 & 76.54 ± 0.21 & 76.00 ± 1.23 & 77.39 ± 0.19 & 43.71 ± 2.25 & 79.69 ± 0.33 \\
FAF+2 & \textbf{82.30 ± 0.31} & 89.76 ± 0.47 & \textbf{78.38 ± 0.56} & 78.43 ± 0.82 & \textbf{78.16 ± 0.30} & 44.22 ± 2.13 & 80.07 ± 0.40 \\
FAF+4 & 79.24 ± 0.52 & \textbf{90.01 ± 0.50} & 77.20 ± 0.45 & \textbf{78.54 ± 1.51} & 76.65 ± 0.41 & \textbf{45.09 ± 2.21} & \textbf{80.18 ± 0.59} \\
FAF+8 & 71.50 ± 0.82 & 89.21 ± 0.75 & 74.50 ± 0.34 & 77.85 ± 2.16 & 74.34 ± 0.92 & 44.17 ± 2.03 & 79.88 ± 0.70 \\
\bottomrule
\end{tabular}
}
\end{table*}

\begin{table*}[h!]
\caption{Validation accuracy for increasing number of concatenated hops in FAF$_4$.}
\label{tab:locallayersval}
\centering\resizebox{\linewidth}{!}{
\begin{tabular}{lccccccc}
\toprule
Dataset & computer & photo & ratings & chameleon & citeseer & coauthor-cs & coauthor-physics \\
\midrule
FAF+0 & 87.89 ± 0.13 & 93.33 ± 0.07 & 48.98 ± 0.72 & 41.43 ± 1.77 & 53.80 ± 0.82 & 93.70 ± 0.07 & 95.89 ± 0.02 \\
FAF+1 & 92.53 ± 0.08 & 96.14 ± 0.07 & 54.17 ± 0.14 & 46.91 ± 1.43 & 65.52 ± 0.64 & 94.84 ± 0.08 & \textbf{96.83 ± 0.01} \\
FAF+2 & 92.99 ± 0.08 & 96.23 ± 0.08 & 55.02 ± 0.67 & 47.30 ± 1.76 & \textbf{67.28 ± 0.64} & \textbf{94.93 ± 0.07} & 96.63 ± 0.02 \\
FAF+4 & \textbf{93.04 ± 0.10} & \textbf{96.23 ± 0.04} & 55.08 ± 0.19 & 48.42 ± 2.22 & 50.52 ± 3.84 & 94.88 ± 0.07 & 96.63 ± 0.04 \\
FAF+8 & 92.96 ± 0.13 & 96.21 ± 0.07 & \textbf{55.26 ± 0.30} & \textbf{48.74 ± 1.69} & 40.72 ± 2.26 & 94.92 ± 0.09 & OOM \\
\bottomrule
\end{tabular}
}
\medskip

\centering\resizebox{\linewidth}{!}{
\begin{tabular}{lccccccc}
\toprule
Dataset & cora & minesweeper & pubmed & questions & roman-empire & squirrel & wikics \\
\midrule
FAF+0 & 62.68 ± 1.15 & 51.12 ± 0.93 & 71.12 ± 0.52 & 71.58 ± 1.46 & 66.28 ± 0.27 & 40.57 ± 0.92 & 74.86 ± 0.33 \\
FAF+1 & 82.16 ± 0.33 & 87.65 ± 0.47 & 78.24 ± 0.36 & 77.44 ± 1.07 & 77.36 ± 0.55 & 47.26 ± 1.31 & 81.37 ± 0.51 \\
FAF+2 & \textbf{82.84 ± 0.38} & 89.48 ± 1.08 & 78.52 ± 0.41 & 79.71 ± 0.86 & \textbf{78.66 ± 0.19} & 47.18 ± 1.44 & \textbf{81.81 ± 0.46} \\
FAF+4 & 81.80 ± 0.35 & \textbf{89.63 ± 1.02} & \textbf{79.08 ± 0.36} & 79.67 ± 0.89 & 77.48 ± 0.17 & 47.61 ± 1.43 & 81.73 ± 0.53 \\
FAF+8 & 74.28 ± 0.23 & 89.10 ± 1.03 & 76.80 ± 0.20 & \textbf{79.94 ± 0.88} & 75.08 ± 0.23 & \textbf{47.95 ± 1.36} & 81.58 ± 0.63 \\
\bottomrule
\end{tabular}
}
\end{table*}

\textbf{Linear vs.\ MLP classifier.}\quad
In Table~\ref{tab:logreg} we show the performance of two classifiers on the same FAF$_4$ features: one linear layer and a multilayer perceptron\textemdash this being our choice for other experiments.

\newpage

\begin{table*}[h!]
\caption{Comparison of 1 linear layer (1L) versus multiple layers (MLP) as the classifier over FAF$_4$.}
\label{tab:logreg}
\centering\resizebox{\linewidth}{!}{
\begin{tabular}{lccccccc}
\toprule
Dataset & computer & photo & ratings & chameleon & citeseer & coauthor-cs & coauthor-physics \\
\midrule
FAF+MLP & 93.75 ± 0.04 & {96.54 ± 0.13} & 54.42 ± 0.45 & 42.96 ± 2.45 & 69.42 ± 1.32 & 95.33 ± 0.20 & 96.96 ± 0.09 \\

FAF+1L & {92.78 ± 0.06} & {96.19 ± 0.10} & {47.80 ± 0.23} & {43.97 ± 3.18} & {67.80 ± 1.22} & {94.07 ± 0.06} & {96.82 ± 0.11} \\

\bottomrule
\end{tabular}
}
\medskip 

\centering\resizebox{\linewidth}{!}{
\begin{tabular}{lccccccc}
\toprule
Dataset & cora & minesweeper & pubmed & questions & roman-empire & squirrel & wikics \\
\midrule
FAF+MLP & 81.44 ± 0.38 & 90.01 ± 0.51 & 77.20 ± 0.45 & {78.69 ± 0.50} & 78.11 ± 0.38 & 44.02 ± 2.08 & 80.25 ± 0.34 \\
FAF+1L & {81.10 ± 0.91} & {89.45 ± 0.65} & {77.00 ± 0.60} & {77.87 ± 1.05} & {74.62 ± 0.03} & {44.14 ± 2.90} & {79.77 ± 0.96} \\

\bottomrule
\end{tabular}
}
\end{table*}

\textbf{Last hop.}\quad
In Table~\ref{tab:lasthop} we ablate on using only the last hop as features to an MLP, or using the last hop concatenated to the original features. This mimics the choice of directly freezing a GNN and using its output as features to an (often linear) classifier. In almost all cases, having all hops is more performant, despite the increase in input size.

\begin{table*}[h!]
\caption{Using only the last hop ($H^{K}$), that and original features ($H^{0\oplus K}$), and all hops (FAF$_4$).}
\label{tab:lasthop}
\centering\resizebox{\linewidth}{!}{
\begin{tabular}{lccccccc}
\toprule
Dataset & computer & photo & ratings & chameleon & citeseer & coauthor-cs & coauthor-physics \\
\midrule
FAF$_4$ & 93.75 ± 0.04 & {96.54 ± 0.13} & 54.42 ± 0.45 & 42.96 ± 2.45 & 69.42 ± 1.32 & 95.33 ± 0.20 & 96.96 ± 0.09 \\
$H^{0\oplus K}$ & {93.13 ± 0.15} & {96.45 ± 0.04} & {54.39 ± 0.56} & 40.33 ± 4.49 & 68.48 ± 1.04 & {95.43 ± 0.33} & {96.82 ± 0.13} \\
$H^{K}$ & 92.68 ± 0.11 & 92.42 ± 0.11 & 48.87 ± 0.36 & {44.28 ± 2.63} & {68.64 ± 1.24} & 93.08 ± 0.23 & 96.33 ± 0.11 \\
\bottomrule
\end{tabular}
}
\medskip 

\centering\resizebox{\linewidth}{!}{
\begin{tabular}{lccccccc}
\toprule
Dataset & cora & minesweeper & pubmed & questions & roman-empire & squirrel & wikics \\
\midrule
FAF$_4$ & 81.44 ± 0.38 & 90.01 ± 0.51 & 77.20 ± 0.45 & {78.69 ± 0.50} & 78.11 ± 0.38 & 44.02 ± 2.08 & 80.25 ± 0.34 \\
$H^{0\oplus K}$ & {81.38 ± 0.59} & {73.28 ± 0.94} & {77.04 ± 0.27} & {77.80 ± 0.92} & {75.78 ± 0.05} & {43.07 ± 2.13} & {79.24 ± 0.40} \\
$H^{K}$ & 81.20 ± 0.97 & 69.96 ± 1.72 & 77.00 ± 0.64 & 77.66 ± 1.44 & 54.11 ± 0.42 & 42.15 ± 1.64 & 77.77 ± 0.77 \\

\bottomrule
\end{tabular}
}
\end{table*}

\textbf{KA reducer.}\quad
Table~\ref{tab:fafka} shows the last epoch training accuracy, best epoch validation accuracy, and corresponding test accuracy when using the Kolmogorov-Arnold function $\Phi$ as a {reducer for FAF}. Following \citet{Corso2020Principal}, we make it act on multisets by sorting, which we fix by the given data order.

\begin{table*}[h!] 
\caption{Last epoch training accuracy, best epoch validation accuracy, and corresponding test accuracy  of the KA function $\Phi$ from Theorem~\ref{thm:ka}.}
\label{tab:fafka}
\centering\resizebox{\linewidth}{!}{
\begin{tabular}{lccccccc}
\toprule
Dataset & computer & photo & ratings & chameleon & citeseer & coauthor-cs & coauthor-physics \\
\midrule
FAF$_{\rm {KA}}$ (train) & 94.73 ± 0.12 & 99.65 ± 0.09 & 99.92 ± 0.02 & 96.35 ± 3.44 & 70.17 ± 41.23 & 99.49 ± 0.22 & 99.84 ± 0.28 \\

FAF$_{\rm {KA}}$ (val) & 87.88 ± 0.19 & 93.40 ± 0.07 & 51.97 ± 0.12 & 41.65 ± 1.91 & 55.76 ± 1.07 & 93.66 ± 0.08 & 95.90 ± 0.05 \\
FAF$_{\rm {KA}}$ (test)& {87.56 ± 0.42} & {93.73 ± 0.07} & {51.46 ± 0.50} & {35.96 ± 3.89} & {56.18 ± 1.50} & {93.70 ± 0.08} & {95.97 ± 0.00} \\

\bottomrule
\end{tabular}
}
\medskip 

\centering\resizebox{\linewidth}{!}{
\begin{tabular}{lccccccc}
\toprule
Dataset & cora & minesweeper & pubmed & questions & roman-empire & squirrel & wikics \\
\midrule
FAF$_{\rm {KA}}$ (train) & 14.29 ± 0.00 & 51.68 ± 0.44 & 33.33 ± 0.00 & 99.85 ± 0.09 & 86.41 ± 0.23 & 39.30 ± 2.30 & 97.64 ± 1.09 \\
FAF$_{\rm {KA}}$ (val) & 29.56 ± 0.79 & 51.61 ± 0.42 & 42.80 ± 0.76 & 74.28 ± 1.95 & 80.45 ± 0.25 & 40.43 ± 1.10 & 77.38 ± 0.88 \\
FAF$_{\rm {KA}}$ (test) & {29.78 ± 0.91} & {52.21 ± 0.81} & {41.06 ± 3.63} & {71.82 ± 0.36} & {80.33 ± 0.47} & {38.62 ± 1.24} & {76.31 ± 0.77} \\
\bottomrule
\end{tabular}
}
\end{table*}

\textbf{Rewiring augmentations.}\quad
In Table~\ref{tab:rewiring} we include results on rewiring the input graph by deleting edges based on pairwise cosine similarity, as mentioned in \S~\ref{s:advantages}. We use mean as reducer for the original and the rewired aggregation. Coauthor-Physics is not included, as it has a large amount of original features. REW includes hop-wise features where all negative similarity neighbors are set to 0. SP includes hop-wise features where positive and negative similarity neighbors are aggregated in different features and concatenated together. Rewiring methods can provide extra signal for the tasks and mitigate fundamental issues like over-squashing and over-smoothing, and are combinable with FAFs.

\newpage

\begin{table*}[h!]
\caption{Feature augmentations based on similarity-based rewiring (REW)~\citep{rubio-madrigal2025gnns} and computational-graph splitting (SP)~\citep{pmlr-v269-roth25a}. Comparison of test accuracy to FAF$_{\rm {mean}}$ and to their combination.}
\label{tab:rewiring}

\centering\resizebox{\linewidth}{!}{%
\begin{tabular}{lccccccc}
\toprule
Dataset & computer & photo & ratings & chameleon & citeseer & coauthor-cs & cora \\
\midrule
FAF$_{\rm mean}$ & {94.01 ± 0.21} & \underline{96.71 ± 0.16} & 53.12 ± 0.44 & 43.21 ± 2.24 & 66.82 ± 1.74 & \underline{95.37 ± 0.17} & \textbf{82.80 ± 0.70} \\
REW$_{\rm mean}$ & \underline{94.22 ± 0.15} & 96.58 ± 0.23 & \underline{54.43 ± 0.12} & 38.54 ± 1.94 & \underline{67.96 ± 0.82} & 95.27 ± 0.09 & 80.88 ± 0.80 \\ 
SP$_{\rm mean}$ & 94.10 ± 0.23 & \textbf{96.75 ± 0.36} & \textbf{54.60 ± 0.38} & 39.77 ± 3.20 & 67.38 ± 1.74 & \textbf{95.46 ± 0.04} & 80.38 ± 0.47 \\
FAF$_{\rm mean}$+REW$_{\rm mean}$ & 94.16 ± 0.17 & {96.64 ± 0.16} & 53.48 ± 0.28 & \textbf{44.31 ± 3.82} & 67.58 ± 1.07 & 95.34 ± 0.21 & \underline{82.68 ± 0.98} \\
FAF$_{\rm mean}$+SP$_{\rm mean}$ & \textbf{94.35 ± 0.04} & 96.69 ± 0.04 & 53.48 ± 0.67 & \underline{43.48 ± 3.71} & \textbf{68.04 ± 1.30} & 95.29 ± 0.16 & 82.38 ± 0.64 \\
\bottomrule
\end{tabular}%
}

\medskip

\centering\resizebox{\linewidth}{!}{%
\begin{tabular}{lcccccc}
\toprule
Dataset & minesweeper & pubmed & questions & roman-empire & squirrel & wikics \\
\midrule

FAF$_{\rm mean}$ & \textbf{90.00 ± 0.39} & 79.88 ± 0.92 & \underline{76.83 ± 1.19} & 76.36 ± 0.55 & 42.44 ± 1.73 & 79.61 ± 0.56 \\
REW$_{\rm mean}$ & 59.57 ± 1.88 & 79.74 ± 0.59 & 75.65 ± 0.28 & 75.75 ± 0.42 & 38.24 ± 1.73 & 80.44 ± 0.37 \\
SP$_{\rm mean}$ & 59.44 ± 2.27 & \underline{80.06 ± 0.63} & 75.83 ± 0.54 & 75.51 ± 0.16 & 38.40 ± 2.11 & \underline{80.54 ± 0.37} \\
FAF$_{\rm mean}$+REW$_{\rm mean}$ & {89.71 ± 0.31} & \textbf{80.16 ± 0.42} & \textbf{77.15 ± 1.25} & \textbf{77.61 ± 0.54} & \textbf{43.25 ± 1.86} & \textbf{80.56 ± 0.52} \\
FAF$_{\rm mean}$+SP$_{\rm mean}$ & \underline{89.71 ± 0.42} & 79.96 ± 0.71 & 76.58 ± 1.19 & \underline{77.18 ± 0.26} & \underline{42.54 ± 1.65} & 80.41 ± 0.67 \\
\bottomrule
\end{tabular}%
}

\end{table*}

\textbf{SHAP on other datasets.}\quad
In Figure~\ref{fig:shap} we show two more plots of feature importance using SHAP \citep{shap} for Pubmed and Amazon-Ratings, on the MLP over single-{reducer} FAFs. Features are sorted by global importance and broken down over the different hops by color. While the implementation of SHAP on MLPs used (GradientExplainer) relies on local linearization and often assumes input feature independence, the explanations still reveal informative qualitative patterns. In Pubmed, feature 346 is most important at hops 1 and 2, and remains second at hops 0 and 4, whereas the most important base feature (205) contributes little at other hops. 
By contrast, in Amazon-Ratings, importance is more evenly distributed across features and hops.

\begin{figure*}[h!] 
    \centering
    \includegraphics[width=0.9\linewidth]{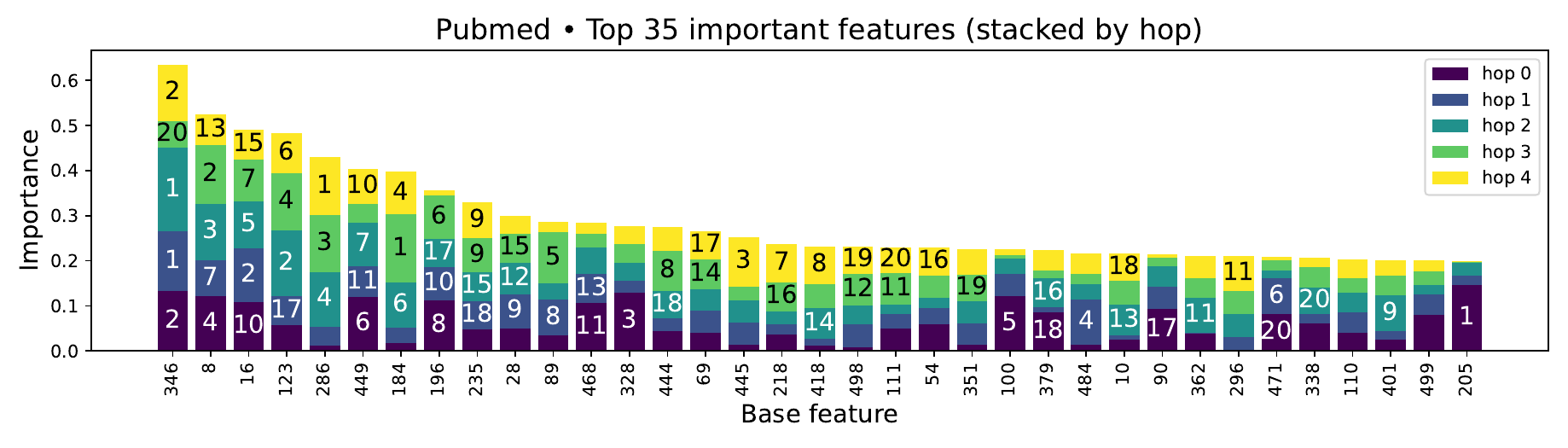}
    
    \includegraphics[width=0.9\linewidth]{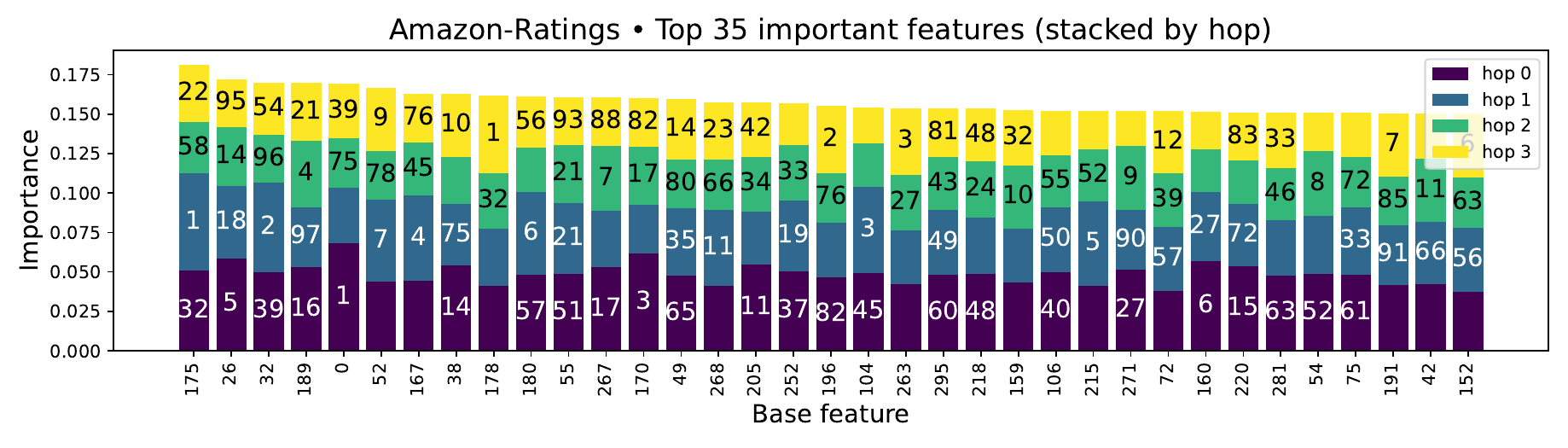}
    
    \caption{SHAP feature importance for Pubmed and Amazon-Ratings. The base features are ranked according to the sum of their importance values across hops. Numbers on the stacked bars indicate the ranking of that particular feature on that particular hop. }
    \label{fig:shap}
\end{figure*}

\newpage

\section{Comparison to Graph Echo State Networks}\label{app:gesn}

{
Graph Echo State Networks (GESN) \citep{gesn} compute label-independent node embeddings via fixed “reservoir” layers, one per hop, followed by a linear readout. We compare FAFs to this approach in Table~\ref{tab:gesn} as a representative previously proposed simplification of GNNs. We do not include Coauthor-Physics or {Questions}, as GESN exceeds memory capacity on them. 
We use the public implementation at \url{https://github.com/dtortorella/graph-esn}, keeping most GESN-specific hyperparameters as in their example. We set the ``depth" (local layers) and ``hidden units" (hidden channels) as in the best GCN. Each layer is given 10 minutes to compute its embedding.
As the classifier, we replace the original linear readout with a MLP of the same architecture as FAFs to provide a fair comparison of the role of the embeddings. For the MLP hyperparameters, we try the best GCN and the best FAF configurations. Since our different FAFs could share MLP hyperparameters, we expected them to transfer well, but in this case the GCN ones perform better. As shown, FAFs seem to be more suitable for the tested benchmark tasks.
}

\begin{table*}[h!]
\caption{Test accuracy of GESN \citep{gesn} against FAFs.}
\label{tab:gesn}
\centering
\begin{tabular}{lccccccc}
\toprule
Dataset & computer & photo & ratings & chameleon & citeseer & coauthor-cs \\ \midrule
FAF$_{\rm best val}$ & 94.01 ± 0.21 & 96.54 ± 0.13 & 55.09 ± 0.24 & 42.96 ± 2.45 & 70.48 ± 1.24 & 95.37 ± 0.17 \\
GESN+MLP & 90.80 ± 0.10 & 92.72 ± 0.27 & 50.34 ± 0.26 & 41.64 ± 3.63 & 41.44 ± 0.34 & 89.99 ± 0.09 & \\
\bottomrule
\end{tabular}

\medskip 

\begin{tabular}{lccccccc}
\toprule
Dataset & cora & minesweeper & pubmed & roman-empire & squirrel & wikics \\
\midrule
FAF$_{\rm best val}$ & 82.84 ± 0.63 & 90.00 ± 0.39 & 80.96 ± 1.06 & 78.11 ± 0.38 & 44.59 ± 1.62 & 80.25 ± 0.34 \\
GESN+MLP &65.78 ± 0.26 & 50.93 ± 1.28 & 64.98 ± 1.57 &  11.76 ± 0.38 & 36.58 ± 1.18 & 73.98 ± 0.85 \\

\bottomrule
\end{tabular}
\end{table*}

\section{Preliminary Results on Other Benchmarks}\label{app:graphland}
{
We include preliminary results for FAFs on the GraphLand bechmark \citep{bazhenov2025graphland}. We do \textit{not} perform the full hyperparameter sweep, therefore we indicate FAFs with an asterisk (*), as there could be better performing versions. We copy baselines from the original paper: MLP, MLP-NFA (one-hop FAFs), GCN and GAT. 
For FAFs, we only report results for mean+std and mean aggregations. 
We choose the best validation hyperparameters we have been able to find so far (shown in Table~\ref{tab:graphlandhyps}) and report the resulting test accuracy 
in Table~\ref{tab:graphland}. While we do not yet achieve the performance of GATs, we approach that of GCNs, and we improve upon MLPs and NFA.}

\begin{table*}[h!]
    \centering
    \caption{Test accuracy of 4 GraphLand datasets (averaged over 10 runs).}\label{tab:graphland}
    \begin{tabular}{lcccc}\toprule
        Dataset & artnet-exp & hm-categories & tolokers-2 & pokec-regions \\ \midrule
        ResMLP & 35.07 ± 2.34 & 37.72 ± 0.18 & 41.16 ± 1.13 & 4.88 ± 0.01 \\
        ResMLP-NFA & 38.25 ± 0.56 & 48.72 ± 0.38 & 48.14 ± 1.40 & 8.05 ± 0.03 \\
        GCN & 43.09 ± 0.38 & 61.70 ± 0.35 & 51.32 ± 0.96 & 34.96 ± 0.38 \\ 
        GAT & 46.62 ± 0.32 & 67.96 ± 0.33 & 53.78 ± 1.34 & 46.17 ± 0.32 \\ \midrule
        FAF*$_{\rm mean,std}$ & 41.56 ± 0.26 & 59.50 ± 0.15 & 52.74 ± 0.53 & 28.44 ± 0.22 \\
        FAF*$_{\rm mean}$ & 39.25 ± 0.38 & 54.50 ± 0.13 & 50.72 ± 0.38 & 31.23 ± 0.18 \\ \bottomrule
    \end{tabular}
\end{table*}

\begin{table*}[h!]
    \centering
    \caption{Best hyperparameters found so far for FAFs on GraphLand datasets.}\label{tab:graphlandhyps}
\centering
    \begin{tabular}{lccccccc}\toprule
        ~ & dropout & lr & bn & hidden channels & weight decay & local layers & mlp layers \\ \midrule
        artnet-exp & 0.7 & 0.01 & 1 & 256 & 0.01 & 3 & 2 \\
        hm-categories & 0.5 & 0.001 & 1 & 512 & 0.0005 & 3 & 3 \\
        tolokers-2 & 0.3 & 0.0001 & 1 & 512 & 5e-05 & 3 & 5 \\
        pokec-regions & 0 & 0.001 & 1 & 512 & 0.001 & 3 & 3 \\ \bottomrule
    \end{tabular}
\end{table*}

We also include preliminary results on GraphBench \citep{stoll2026graphbenchnextgenerationgraphlearning}. We include preliminary results on their first two tasks (test MAE, lower is better). Their baselines are not tuned so neither do we, but we use the hyperparameters from GraphLand's tasks. Again, an asterisk (*) denotes there could be better FAFs. We show that FAF*$_\text{mean}$ performs better than the best reported model (MPGNN).

\begin{table*}[h!]
    \centering
    \caption{Test MAE of 2 GraphBench datasets (lower is better).}\label{tab:graphland}
    \begin{tabular}{lcc}\toprule
        Dataset & quotes & replies  \\ \midrule
        GNN & 0.768 ± 0.002 & 0.694 ± 0.002 \\
        FAF*$_\text{mean}$ & 0.661 ± 0.002 & 0.579 ± 0.001
        \\ \bottomrule
    \end{tabular}
\end{table*}

Newer benchmarks have not yet been as widely studied as the benchmarks we focus on, and the limits of state-of-the-art models on them are not yet well understood. For instance, it is still unclear whether classic GNNs rival more complex architectures on these datasets. For this reason, we do not perform the same extensive hyperparameter tuning. A more thorough search could yield a better-performing FAF pipeline, so, for instance, we cannot definitely conclude that GraphLand requires learned aggregation; especially so, as we can already rival non-attention-based models. This is consistent with our broader finding that datasets requiring learned aggregation tend to be those that benefit from significantly deeper architectures (more than 10 layers), a pattern we have not observed in the well-studied shallow benchmarks. 

We append these preliminary results primarily to motivate the community to include well-tuned FAF baselines alongside graph-agnostic baselines in future benchmark evaluations, as they increase performance considerably from MLPs. In terms of tuning, we note, for instance, that the hyperparameter tuning of GraphLand includes Dropout of up to 0.2, while for MLPs we have found larger values (up to 0.7) to perform better, as well as BatchNorm in exchange of LayerNorm. Our secondary message is that graph-agnostic baselines require well-chosen hyperparameter ranges in order to compete fairly, which they frequently achieve when fed with our well-constructed features.

\section{Original results from \citep{luo2024classic}}

For convenience, in Tables~\ref{tab:tab2}, \ref{tab:tab3} we present the original results tables from the benchmark we evaluate against \citep{luo2024classic}. This shows that the classic GNNs also rival graph transformers and heterophily-aware models. The numbers we report may differ within error, as we re-run their baselines.

\definecolor{darkgreen}{rgb}{0.0, 0.5, 0.0}
\definecolor{peach}{rgb}{1.0, 0.85, 0.7}
\definecolor{mediumgreen}{RGB}{60,179,113}
\definecolor{customcyan}{RGB}{10, 204, 0} 
\definecolor{tealblue}{RGB}{0, 132, 194}
\definecolor{darkorange}{RGB}{220, 100, 0}

\begin{table*}[t]
    \centering
    \caption{\citep{luo2024classic} Node classification results over homophilous graphs (\%). The top \textbf{\textcolor{customcyan}{$\mathbf{1^{st}}$}}, \textbf{\textcolor{tealblue!90}{$\mathbf{2^{nd}}$}} and \textbf{\textcolor{darkorange!90}{$\mathbf{3^{rd}}$}} results are highlighted. }
    \setlength\tabcolsep{3pt}
    \resizebox{\linewidth}{!}{
    \begin{tabular}{l|llllllll}
        \toprule
            & Cora   &CiteSeer   & PubMed      
            &Computer &Photo &CS &Physics &WikiCS 
            \\
         \midrule 
        \# nodes    & 2,708   & 3,327     & 19,717   &13,752 &7,650 &18,333 &34,493 &11,701
        \\
        \# edges & 5,278 & 4,732 & 44,324 &245,861 &119,081 &81,894 &247,962 &216,123
        \\
         Metric & Accuracy$\uparrow$  & Accuracy$\uparrow$
         & Accuracy$\uparrow$ &Accuracy$\uparrow$ &Accuracy$\uparrow$ &Accuracy$\uparrow$ &Accuracy$\uparrow$ &Accuracy$\uparrow$ 
         \\
        \midrule %
        
       GraphGPS & 82.84 {\tiny{± 1.03}} & 72.73 {\tiny{± 1.23}} & 79.94 {\tiny{± 0.26}} & 91.19 {\tiny{± 0.54}}&  95.06 {\tiny{± 0.13}}&  93.93 {\tiny{± 0.12}}&  97.12 {\tiny{± 0.19}} &  78.66 {\tiny{± 0.49}} \\
        \rowcolor{gray!20} 
        \textbf{GraphGPS$^*$} & 83.87 {\tiny{± 0.96}} & \textcolor{tealblue!90}{\textbf{72.73}} {\tiny{± 1.23}} & 79.94 {\tiny{± 0.26}} & 91.79 {\tiny{± 0.63}}&  94.89 {\tiny{± 0.14}}&  94.04 {\tiny{± 0.21}}&  96.71 {\tiny{± 0.15}} &  78.66 {\tiny{± 0.49}} \\
        NAGphormer & 82.12 {\tiny{± 1.18}}  & 71.47 {\tiny{± 1.30}}  & 79.73 {\tiny{± 0.28}}  &91.22 {\tiny{± 0.14}} &  95.49 {\tiny{± 0.11}} & 95.75 {\tiny{± 0.09}} & 97.34 {\tiny{± 0.03}} & 77.16 {\tiny{± 0.72}}  \\
        \rowcolor{gray!20}
        \textbf{NAGphormer$^*$} & 80.92 {\tiny{± 1.17}} & 70.59 {\tiny{± 0.89}} & 80.14 {\tiny{± 1.06}} & 91.69 {\tiny{± 0.30}} & 96.14 {\tiny{± 0.16}} & 95.85 {\tiny{± 0.16}} & \textcolor{tealblue!90}{\textbf{97.35}} {\tiny{± 0.12}} & 77.92 {\tiny{± 0.93}} \\
        Exphormer & 82.77 {\tiny{± 1.38}}  & 71.63 {\tiny{± 1.19}} & 79.46 {\tiny{± 0.35}}  &91.47 {\tiny{± 0.17}}  & 95.35 {\tiny{± 0.22}} & 94.93 {\tiny{± 0.01}} & 96.89 {\tiny{± 0.09}} & 78.54 {\tiny{± 0.49}} \\
        \rowcolor{gray!20}
        \textbf{Exphormer$^*$} & 83.29 {\tiny{± 1.36}} & 71.85 {\tiny{± 1.11}} & 79.67 {\tiny{± 0.73}} & 91.80 {\tiny{± 0.35}} & 95.69 {\tiny{± 0.39}} & 95.92 {\tiny{± 0.25}} & 97.06 {\tiny{± 0.13}} & 79.38 {\tiny{± 0.62}} \\
        GOAT       & 83.18 {\tiny{± 1.27}} & 71.99 {\tiny{± 1.26}}  & 79.13 {\tiny{± 0.38}} &  90.96 {\tiny{± 0.90}} & 92.96 {\tiny{± 1.48}} & 94.21 {\tiny{± 0.38}} & 96.24 {\tiny{± 0.24}} & 77.00 {\tiny{± 0.77}}\\
        \rowcolor{gray!20}
        \textbf{GOAT$^*$} & 83.26 {\tiny{± 1.24}} & 72.21 {\tiny{± 1.29}} & 80.06 {\tiny{± 0.67}} & 92.29 {\tiny{± 0.37}} & 94.33 {\tiny{± 0.21}} & 93.81 {\tiny{± 0.19}} & 96.47 {\tiny{± 0.16}} & 77.96 {\tiny{± 0.63}}\\
        NodeFormer & 82.20 {\tiny{± 0.90}}  & 72.50 {\tiny{± 1.10}} & 79.90 {\tiny{± 1.00}}  & 86.98 {\tiny{± 0.62}} & 93.46 {\tiny{± 0.35}} & 95.64 {\tiny{± 0.22}} & 96.45 {\tiny{± 0.28}} & 74.73 {\tiny{± 0.94}} \\
        \rowcolor{gray!20}
        \textbf{NodeFormer$^*$} & 82.73 {\tiny{± 0.75}}  & 72.37 {\tiny{± 1.20}} & 79.59 {\tiny{± 0.92}}  & 87.29 {\tiny{± 0.58}} & 93.43 {\tiny{± 0.56}} & 95.69 {\tiny{± 0.27}} & 96.48 {\tiny{± 0.34}} & 75.13 {\tiny{± 0.93}} \\
        SGFormer    & 84.50 {\tiny{± 0.80}} & 72.60 {\tiny{± 0.20}} & 80.30 {\tiny{± 0.60}}  & 91.99 {\tiny{± 0.76}} & 95.10 {\tiny{± 0.47}} & 94.78 {\tiny{± 0.20}} & 96.60 {\tiny{± 0.18}} & 73.46 {\tiny{± 0.56}}
        
        \\
        \rowcolor{gray!20}
        \textbf{SGFormer$^*$}    & \textcolor{tealblue!90}{\textbf{84.82}} {\tiny{± 0.85}} & \textcolor{darkorange!90}{\textbf{72.72}} {\tiny{± 1.15}} & \textcolor{tealblue!90}{\textbf{80.60}} {\tiny{± 0.49}}  & 92.42 {\tiny{± 0.66}} & 95.58 {\tiny{± 0.36}} & 95.71 {\tiny{± 0.24}} & 96.75 {\tiny{± 0.26}} & 80.05 {\tiny{± 0.46}}
        
        \\
        Polynormer & 83.25 {\tiny{± 0.93}}  & 72.31 {\tiny{± 0.78}} & 79.24 {\tiny{± 0.43}} & 93.68 {\tiny{± 0.21}}  & 96.46 {\tiny{± 0.26}} & 95.53 {\tiny{± 0.16}}  & 97.27 {\tiny{± 0.08}} & 80.10 {\tiny{± 0.67}}  \\
        \rowcolor{gray!20}
        \textbf{Polynormer$^*$} & 83.43 {\tiny{± 0.89}}  & 72.19 {\tiny{± 0.83}} & 79.35 {\tiny{± 0.73}} & \textcolor{darkorange!90}{\textbf{93.78}} {\tiny{± 0.10}}  & \textcolor{darkorange!90}{\textbf{96.57}} {\tiny{± 0.23}} & 95.42 {\tiny{± 0.19}}  & 97.18 {\tiny{± 0.11}} & 80.26 {\tiny{± 0.92}}  \\
        \midrule 
        GCN        & 81.60 {\tiny{± 0.40}} & 71.60 {\tiny{± 0.40}} & 78.80 {\tiny{± 0.60}} & 89.65 {\tiny{± 0.52}}  &92.70 {\tiny{± 0.20}}  &92.92 {\tiny{± 0.12}}  &96.18 {\tiny{± 0.07}}  &77.47 {\tiny{± 0.85}}\\
        \rowcolor{gray!20}
        \textbf{GCN$^*$}        & \textcolor{customcyan}{\textbf{85.10}} {\tiny{± 0.67}} \textbf{3.50$\uparrow$} & \textcolor{customcyan}{\textbf{73.14}} {\tiny{± 0.67}} \textbf{1.54$\uparrow$} & \textcolor{customcyan}{\textbf{81.12}} {\tiny{± 0.52}} \textbf{2.32$\uparrow$} & \textcolor{tealblue!90}{\textbf{93.99}} {\tiny{± 0.12}} \textbf{4.34$\uparrow$}  &{{96.10}} {\tiny{± 0.46}} \textbf{3.40$\uparrow$}  &\textcolor{darkorange!90}{\textbf{96.17}} {\tiny{± 0.06}} \textbf{3.25$\uparrow$}  &\textcolor{customcyan}{\textbf{97.46}} {\tiny{± 0.10}} \textbf{1.28$\uparrow$}  &\textcolor{darkorange!90}{\textbf{80.30}}  {\tiny{± 0.62}} \textbf{2.83$\uparrow$}\\

        \midrule 
        GraphSAGE & 82.68 {\tiny{± 0.47}} & 71.93 {\tiny{± 0.85}} & 79.41 {\tiny{± 0.53}} &91.20 {\tiny{± 0.29}} &94.59 {\tiny{± 0.14}} &93.91 {\tiny{± 0.13}} &96.49 {\tiny{± 0.06}} & 74.77 {\tiny{± 0.95}}\\
        \rowcolor{gray!20}
        \textbf{GraphSAGE$^*$} & 83.88 {\tiny{± 0.65}} \textbf{1.20$\uparrow$} & 72.26 {\tiny{± 0.55}} \textbf{0.33$\uparrow$} & 79.72 {\tiny{± 0.50}} \textbf{0.31$\uparrow$} &93.25 {\tiny{± 0.14}} \textbf{2.05$\uparrow$} &\textcolor{customcyan}{\textbf{96.78}} {\tiny{± 0.23}} \textbf{2.19$\uparrow$} &\textcolor{customcyan}{\textbf{96.38}} {\tiny{± 0.11}} \textbf{2.47$\uparrow$} &97.19 {\tiny{± 0.05}} \textbf{0.70$\uparrow$} & \textcolor{tealblue!90}{\textbf{80.69}} {\tiny{± 0.31}} \textbf{5.92$\uparrow$} \\

        \midrule 
        GAT        & 83.00 {\tiny{± 0.70}} & 72.10 {\tiny{± 1.10}} & 79.00 {\tiny{± 0.40}} & 90.78 {\tiny{± 0.13}}  &93.87 {\tiny{± 0.11}}  &93.61 {\tiny{± 0.14}}  &96.17 {\tiny{± 0.08}}  &76.91 {\tiny{± 0.82}} \\
        \rowcolor{gray!20}
        \textbf{GAT$^*$}        & \textcolor{darkorange!90}{\textbf{84.46}} {\tiny{± 0.55}} \textbf{1.46$\uparrow$} & 72.22 {\tiny{± 0.84}} \textbf{0.12$\uparrow$} & \textcolor{darkorange!90}{\textbf{80.28}} {\tiny{± 0.64}} \textbf{1.28$\uparrow$} & \textcolor{customcyan}{\textbf{94.09}} {\tiny{± 0.37}} \textbf{3.31$\uparrow$}  &\textcolor{tealblue}{\textbf{96.60}} {\tiny{± 0.33}} \textbf{2.73$\uparrow$} &\textcolor{tealblue!90}{\textbf{96.21}} {\tiny{± 0.14}}  \textbf{2.60$\uparrow$} &\textcolor{darkorange!90}{\textbf{97.25}} {\tiny{± 0.06}} \textbf{1.08$\uparrow$} &\textcolor{customcyan}{\textbf{81.07}}  {\tiny{± 0.54}} \textbf{4.16$\uparrow$} \\
        \bottomrule
    \end{tabular}
    }
    \label{tab:tab2}
\end{table*}

\begin{table*}[t]
    \centering
    {
    \caption{\citep{luo2024classic} Node classification results on heterophilous graphs (\%). The top \textbf{\textcolor{customcyan}{$\mathbf{1^{st}}$}}, \textbf{\textcolor{tealblue!90}{$\mathbf{2^{nd}}$}} and \textbf{\textcolor{darkorange!90}{$\mathbf{3^{rd}}$}} results are highlighted.}
    \vspace{-0.05 in}
    \setlength\tabcolsep{4pt}
\resizebox{\linewidth}{!}{
    \begin{tabular}{l|llllll}
        \toprule
& Squirrel & Chameleon &Amazon-Ratings &Roman-Empire &Minesweeper &Questions \\
\midrule 
\# nodes  & 2223  & 890 &24,492 & 22,662 &10,000 & 48,921 \\
\# edges & 46,998 & 8,854   &93,050 &32,927 & 39,402 & 153,540 \\
 Metric & Accuracy$\uparrow$ & Accuracy$\uparrow$ & Accuracy$\uparrow$ &Accuracy$\uparrow$ &ROC-AUC$\uparrow$&ROC-AUC$\uparrow$  \\
 \midrule 
H2GCN &  35.10 {\tiny{± 1.15}} & 26.75 {\tiny{± 3.64}} & 36.47 {\tiny{± 0.23}} &60.11 \tiny{± 0.52} &89.71 \tiny{± 0.31} &63.59 \tiny{± 1.46} \\ 
CPGNN &  30.04 {\tiny{± 2.03}} & 33.00 {\tiny{± 3.15}} & 39.79 {\tiny{± 0.77}} &63.96 \tiny{± 0.62} &52.03 \tiny{± 5.46} &65.96 \tiny{± 1.95} \\ 
GPRGNN&  38.95 {\tiny{± 1.99}} & 39.93 {\tiny{± 3.30}} & 44.88 {\tiny{± 0.34}} &64.85 \tiny{± 0.27} &86.24 \tiny{± 0.61} &55.48 \tiny{± 0.91} \\ 
FSGNN&  35.92 {\tiny{± 1.32}} & 40.61 {\tiny{± 2.97}}  & 52.74 {\tiny{± 0.83}} &79.92 \tiny{± 0.56} &90.08 \tiny{± 0.70} &\textcolor{darkorange!90}{\textbf{78.86}} \tiny{± 0.92} \\ 
GloGNN&  35.11 {\tiny{± 1.24}} & 25.90 {\tiny{± 3.58}} & 36.89 {\tiny{± 0.14}} &59.63 \tiny{± 0.69} &51.08 \tiny{± 1.23} &65.74 \tiny{± 1.19}  \\ 
 \midrule 
GraphGPS& 39.67 {\tiny{± 2.84}} &40.79 {\tiny{± 4.03}}   & 53.10 {\tiny{± 0.42}} &82.00 \tiny{± 0.61} &90.63 \tiny{± 0.67} &71.73 \tiny{± 1.47}  \\ 
\rowcolor{gray!20}
\textbf{GraphGPS$^*$}& 39.81 {\tiny{± 2.28}} &41.55 {\tiny{± 3.91}}   & 53.27 {\tiny{± 0.66}} &82.72 \tiny{± 0.68} &90.75 \tiny{± 0.89} &72.56 \tiny{± 1.33}  \\ 
NodeFormer&  38.52 {\tiny{± 1.57}} & 34.73 {\tiny{± 4.14}} &  43.86 {\tiny{± 0.35}} &64.49 \tiny{± 0.73}& 86.71 \tiny{± 0.88}& 74.27 \tiny{± 1.46}\\ 
\rowcolor{gray!20}
\textbf{NodeFormer$^*$}&  38.89 {\tiny{± 2.67}} & 36.38 {\tiny{± 3.85}} &  43.79 {\tiny{± 0.57}} &74.83 \tiny{± 0.81}& 87.71 \tiny{± 0.69}& 75.02 \tiny{± 1.61}\\ 
SGFormer&  41.80 {\tiny{± 2.27}} & 44.93 {\tiny{± 3.91}} &  48.01 {\tiny{± 0.49}} &79.10 \tiny{± 0.32}& 90.89 \tiny{± 0.58}& 72.15 \tiny{± 1.31} \\ 
\rowcolor{gray!20}
\textbf{SGFormer$^*$}&  \textcolor{tealblue!90}{\textbf{42.65}} {\tiny{± 2.41}} & \textcolor{tealblue!90}{\textbf{45.21}} {\tiny{± 3.72}} &  54.14 {\tiny{± 0.62}} &80.01 \tiny{± 0.44}& 91.42 \tiny{± 0.41}& 73.81 \tiny{± 0.59} \\ 
Polynormer& 40.87 {\tiny{± 1.96}} &41.82 {\tiny{± 3.45}} & 54.81 {\tiny{± 0.49}}& 92.55 \tiny{± 0.37}& 97.46 \tiny{± 0.36}& 78.92 \tiny{± 0.89} \\ 
\rowcolor{gray!20}
\textbf{Polynormer$^*$}& \textcolor{darkorange!90}{\textbf{41.97}} {\tiny{± 2.14}} &41.97 {\tiny{± 3.18}} & \textcolor{tealblue!90}{\textbf{54.96}} {\tiny{± 0.22}}& \textcolor{customcyan}{\textbf{92.66}} \tiny{± 0.60}& {97.49} \tiny{± 0.48}& \textcolor{tealblue!90}{\textbf{78.94}} \tiny{± 0.78} \\ 
 \midrule 
 GCN & 38.67 {\tiny{± 1.84}} & 41.31 {\tiny{± 3.05}} &48.70 {\tiny{± 0.63}} &73.69 \tiny{± 0.74}&  89.75 \tiny{± 0.52}& 76.09 \tiny{± 1.27}
 \\
 \rowcolor{gray!20}
 \textbf{GCN$^*$} & \textcolor{customcyan}{\textbf{45.01}} {\tiny{± 1.63}} \textbf{6.34$\uparrow$} & \textcolor{customcyan}{\textbf{46.29}} {\tiny{± 3.40}} \textbf{4.98$\uparrow$} & 53.80 {\tiny{± 0.60}} \textbf{5.10$\uparrow$} & \textcolor{tealblue!90}{\textbf{91.27}} {\tiny{± 0.20}} \textbf{17.58$\uparrow$} &  \textcolor{customcyan}{\textbf{97.86}} {\tiny{± 0.24}} \textbf{8.11$\uparrow$} & \textcolor{customcyan}{\textbf{79.02}} {\tiny{± 0.60}} \textbf{2.93$\uparrow$} \\
 \midrule 
GraphSAGE & 36.09 {\tiny{± 1.99}} & 37.77 {\tiny{± 4.14}} &53.63 {\tiny{± 0.39}} &85.74 \tiny{± 0.67} &93.51 \tiny{± 0.57} &76.44 \tiny{± 0.62}\\
\rowcolor{gray!20}
\textbf{GraphSAGE$^*$} & 40.78 {\tiny{± 1.47}} \textbf{4.69$\uparrow$} & \textcolor{darkorange!90}{\textbf{44.81}} {\tiny{± 4.74}} \textbf{7.04$\uparrow$} &\textcolor{darkorange!90}{\textbf{55.40}} {\tiny{± 0.21}} \textbf{1.77$\uparrow$} &\textcolor{darkorange!90}{\textbf{91.06}} {\tiny{± 0.27}} \textbf{5.32$\uparrow$} & \textcolor{tealblue!90}{\textbf{97.77}} {\tiny{± 0.62}} \textbf{4.26$\uparrow$} & {77.21} {\tiny{± 1.28}} \textbf{0.77$\uparrow$} \\

\midrule 
GAT & 35.62 {\tiny{± 2.06}} &39.21 {\tiny{± 3.08}}   &52.70 {\tiny{± 0.62}} &88.75 \tiny{± 0.41} &93.91 \tiny{± 0.35}& 76.79 \tiny{± 0.71}\\ 
\rowcolor{gray!20}
\textbf{GAT$^*$} & 41.73 {\tiny{± 2.07}} \textbf{6.11$\uparrow$} &{44.13} {\tiny{± 4.17}} \textbf{4.92$\uparrow$} &\textcolor{customcyan}{\textbf{55.54}} {\tiny{± 0.51}} \textbf{2.84$\uparrow$} &{90.63} {\tiny{± 0.14}} \textbf{1.88$\uparrow$} &\textcolor{darkorange!90}{\textbf{97.73}} {\tiny{± 0.73}} \textbf{3.82$\uparrow$} & {77.95} {\tiny{± 0.51}} \textbf{1.16$\uparrow$} \\
        \bottomrule
    \end{tabular}
    }
    \label{tab:tab3}}
\end{table*}

\end{document}